\documentclass[twoside]{article}

\usepackage[accepted]{ais_arx}
\usepackage{amsfonts}
\usepackage{mathtools}
\usepackage[thmmarks, amsmath, standard]{ntheorem}
\usepackage[ruled,vlined,noend]{algorithm2e}
\usepackage{graphicx}
\usepackage{soul}
\usepackage[dvipsnames]{xcolor}
\graphicspath{ {./exp_figures/} }
\usepackage{subcaption}

\usepackage[round]{natbib}

\usepackage[colorlinks,
            linkcolor=blue,
            citecolor=blue,
            urlcolor=magenta,
            linktocpage,
            plainpages=false]{hyperref}

\newcommand{\notedone}[1]{}

\DeclareMathOperator*{\E}{\mathsf{E}}
\DeclareMathOperator*{\Var}{\mathsf{Var}}

\DeclareMathOperator*{\argmax}{argmax}
\DeclareMathOperator*{\bo}{\mathcal{O}}
\newcommand{\defeq}{:=}

\newcommand{\ind}[1]{\mathbf{1}\left[ #1 \right]}

\newcommand{\starucb}{Wait-UCB}
\newcommand{\fta}{F2A}
\newcommand{\gtrue}[2]{g_{#1,#2}}
\newcommand{\ghat}[3]{\hat{g}_{#1,#2}(#3)}
\newcommand{\updev}[3]{a_{#1,#2}(#3)}
\newcommand{\npulls}[2]{N_{#1}(#2)}
\newcommand{\gap}[2]{\Delta_{#1,#2}}
\newcommand{\hhat}[3]{\hat{h}_{#1,#2}(#3)}
\newcommand{\confrad}[3]{u_{#1,#2}(#3)}
\newcommand{\reward}[2]{r_{#1}^{(#2)}}
\newcommand{\cost}[2]{c_{#1}^{(#2)}}

\newcommand{\appref}[1]{\ref{#1}} 
\renewcommand{\appref}[1]{\ref*{#1}} 

\begin{document}
\twocolumn[

\aistatstitle{A Farewell to Arms: Sequential Reward Maximization on a Budget with a Giving Up Option}


\aistatsauthor{ P Sharoff \And Nishant A. Mehta \And  Ravi Ganti }

\aistatsaddress{Department of Computer Science \\University of Victoria \\ \texttt{sharoff@uvic.ca} \And Department of Computer Science \\University of Victoria \\ \texttt{nmehta@uvic.ca} \And Google \\ \texttt{gmravi2003@gmail.com} } ]

\begin{abstract}
We consider a sequential decision-making problem where an agent can take one action at a time and each action has a stochastic temporal extent, i.e., a new action cannot be taken until the previous one is finished. Upon completion, the chosen action yields a stochastic reward. The agent seeks to maximize its cumulative reward over a finite time budget, with the option of 
``giving up'' on a current action 
---  hence forfeiting any reward -- in order to choose another action. 
We cast this problem as a variant of the stochastic multi-armed bandits problem with stochastic consumption of resource.
For this problem, we first establish that the optimal arm is the one that maximizes the ratio of the expected reward of the arm to the expected waiting time before the agent sees the reward due to pulling that arm. Using a novel upper confidence bound on 
this ratio, 
we then introduce an upper confidence based-algorithm, \starucb{}, for which we establish logarithmic, problem-dependent regret bound which has an improved dependence on problem parameters compared to previous works. Simulations on various problem configurations comparing \starucb{} against the state-of-the-art algorithms are also presented.
\end{abstract}

\section{Introduction}

In online learning, regret is not always about taking the wrong action. In the real world, an action can take some unknown time to return its reward, and, due to resources being tied up, a learning agent may be unable to take its next action until the previous action is completed. The learner is thus presented with a choice: should it wait until its current action returns a reward or instead bid farewell to the chosen action after a selected waiting time has transpired in order to take a different action. 


Consider, for example, a model of crowdsourcing where an employer submits tasks to a crowdsourcing platform which has a large pool of workers. The employer wishes to have as many tasks as possible completed under a fixed time budget.  
 However, for a given type of task,
the efficiency of workers can vary significantly.
While some omniscient actor could dispatch tasks only to fast workers, a typical employer would not know the efficiency of different workers on the task at hand.
If a randomly selected worker is fast, it can be worthwhile to wait for the worker to complete a job; however, if a worker takes too long to complete a task, it can be advantageous for the employer to terminate the assignment 
and reassign the task to a new worker.
How long should the employer wait for an assigned task to be completed before reassigning the task to a new worker?
This problem fits into a new multi-armed bandit framework, \emph{Farewell to Arms} (\fta{}), that introduces the idea of ``giving up'' on an action (arm) that takes too long to return a reward; for a given arm, we view choosing the waiting time itself as pulling a kind of ``micro-arm''.
As we explain in Section~\ref{sec:application}, other applications include hyperparameter tuning, repeated second-price auctions with participation costs, and computational advertising.

Informally, the \fta{} framework is a variant of the classical stochastic $K$-armed bandit problem framework, but with an added twist of a stochastic resource consumption. For simplicity, consider first the case of the \fta{} framework with only a single arm. Upon pulling the single arm the learning agent gets some stochastic reward, not instantaneously, but only after a stochastic delay\footnote{The stochastic rewards and delays can be dependent.}, say $\tau$. There could be cases where the delay is extremely large. To incorporate this notion, the learning agent is not only required to choose an arm, but is also required to commit to a waiting time $j$. The learning agent receives a reward if $j \geq \tau$, i.e., the agent is willing to wait until the reward arrives, and gets a reward of $0$ otherwise. Hence, the \fta{} framework can be thought of as a 2-level stochastic multi-armed bandit framework with $K$ ``macro-arms'' and (say) $D$ ``micro-arms'' for each of the $K$ macro-arms. The micro-arms capture the willingness of the learning agent to wait and hence bound the total amount of time that can be consumed by a pull, while the macro-arms capture the goodness of a certain arm.

An \fta{} problem can be cast into the more general bandits with knapsack (BwK) framework
\citep{badanidiyuru2013bandits} 
by considering the latter with only a single resource. 
However, whereas \cite{badanidiyuru2013bandits} established worst-case regret bounds for the BwK problem, 
in this work we establish (logarithmic) problem-dependent regret bounds. 
Such logarithmic problem-dependent regret bounds tend to be sharper and better capture problem complexity in the standard stochastic multi-armed bandit case when the gap between the best and the second best arm is large. 
The pioneering work of~\cite{flajolet2015logarithmic} established, for their algorithm UCB-Simplex, logarithmic problem-dependent regret bounds (growing logarithmically in the number of rounds) that also apply to \fta{} problems. 
However, the new algorithm we develop --- \starucb{} --- is based on a rather different upper confidence bound than the one used by UCB-Simplex; moreover, the regret bounds we prove for \starucb{} are often better
than those of \cite{flajolet2015logarithmic} and also those derived by \cite{xia2016budgeted} for their algorithm Budget-UCB. 
Furthermore, we also show that~\starucb{} has better empirical performance than UCB-Simplex and Budget-UCB \citep{xia2016budgeted}. Many important problems (as discussed in Section~\ref{sec:application}) fall in the more specialized \fta{} framework and hence this framework, despite being a special case of the BwK framework, demands a specialized treatment and better algorithms. 


Our core contributions are as follows:
\begin{itemize}
\item We introduce the Farewell to Arms framework.
\item We show in Section~\ref{sec:optimal_policy} that, due to the stochastic consumption of resources in an \fta{} problem, the right quality measure for an arm is the ratio of expected reward to expected waiting time.
\item We derive in Section~\ref{sec:starucb} a novel upper confidence bound for the ratio of mean reward to mean waiting time; using this bound, we design \starucb{}, a new upper confidence-style algorithm.
\item We establish (Section~\ref{sec:regret}) a  logarithmic problem-dependent regret guarantee for \starucb{} which is never worse than $\mathcal{O}((D^3/\Delta) \log T)$, where $\Delta$ is the gap between the aforementioned ratio for the best arm to the best suboptimal arm.
 In important regimes for \fta{} problems (like when mean waiting times for most arms are small), our bound can be $D$ times smaller than the 
regret bounds for UCB-Simplex  \citep{flajolet2015logarithmic} and Budget-UCB \citep{xia2016budgeted}. 

\item We provide (Section~\ref{sec:expt}) a detailed experimental study of \starucb{}, including comparisons to UCB-Simplex and Budget-UCB which suggest that \starucb{} fares better for \fta{} problems.
\end{itemize}

\section{The Farewell to Arms framework}\label{sec:learningproblemsetting}
We now
formally introduce the Farewell to Arms framework and then show how it captures several important applications.

\subsection{A Farewell to Arms game}

In the \fta{} framework, there is a hierarchical set of arms with $K$ macro-arms at the top level and, for each macro-arm, $D$ micro-arms at the next level.
We will index macro-arms with $k \in [K] := \{1, 2, \ldots, K\}$ and micro-arms with $j \in [D]$.
Associated with each macro-arm $k$, there is a joint distribution $Q_k$ over $[0, 1] \times [D]$, for a space of rewards $[0,1]$\footnote{We assume rewards are bounded, in which case it is without loss of generality that we can and will assume that they fall in the unit interval.} and a space of delays $[D]$. A game in the \fta{} framework lasts for $T$ time units and proceeds in epochs. In each epoch $s = 1, 2, \ldots$,
\begin{enumerate}
    \item The learning agent plays a macro/micro-arm pair $i_s := (k,j) \in [K] \times [D]$.
    \item Independently of the learning agent's choice, the stochastic environment draws 
    a potential reward of $V_{k,s} \in [0,1]$ and a delay of $\tau_{k,s} \in [D]$ from distribution $Q_k$.\footnote{We consider a finite number of delay values. The same ideas with minor technical changes can be applied for the case when the number of delay values is finite but the set of delays is not necessarily equal to $[D]$.}
    \item The agent collects rewards $\reward{i_s}{s}:=V_{k,s}\cdot \ind{\tau_{k,s} \leq j}$
 and consumes $\cost{i_s}{s} := \min \{ \tau_{k,s}, j \}$ units of time.
\end{enumerate}
Epochs can be of variable length in the \fta{} framework. This is in contrast to the standard multi-armed bandit framework, where each epoch has unit length. Hence, given a total time budget of $T$ units in the \fta{} framework there can be a variable number of epochs, 
whereas in the multi-armed bandit framework there are exactly $T$ epochs.

The goal of the learning algorithm is to maximize reward within the fixed time budget of $T$.
We will study the \emph{pseudo-regret} of the learning algorithm against the best constant policy: 
the pseudo-regret of a learning algorithm that plays the macro/micro-arm pair sequence $i_1, i_2, \ldots$ is
\begin{align}
R_T :=
    \max_{(k,j) \in [K] \times [D]} \E \left[ \sum_{s=1}^{L_{k,j}} \reward{k,j}{s} \right]
    - \E \left[ \sum_{s=1}^L \reward{i_s}{s} \right] ; \label{eqn:regret}
\end{align}
where $L_{k,j}$ and $L$ are the random stopping times when playing the macro/micro-arm pair sequence $((k,j), (k,j), \ldots, ))$ and $(i_1, i_2, \ldots)$ respectively.
Intuitively, it seems that an optimal (constant) policy is one that maximizes the average reward obtained per unit time.
In Theorem~\ref{thm:exp-cumulative-reward} we show that this is indeed true
and that the \emph{ratio estimator} $\frac{\E [ r ]}{\E [ c ]}$,
where $r$ is the reward and $c$ is the amount of resource consumed when a certain macro/micro-arm pair is pulled,
is indeed the right estimator to optimize for.

\subsection{Applications of the \fta{} framework} \label{sec:application}

In addition to the crowdsourcing example mentioned earlier, many other applications fit into the \fta{} framework. We now present a few of them.

\paragraph{Repeated second price auctions.}  
In a repeated, sealed, second price auction \citep{weed2016online} with participation costs~\citep{gal2007participation,stegeman1996participation,mcafee1987auctions,samuelson1985competitive},  
the goal is to maximize the expected cumulative reward given a budget of $B$ dollars.
In each round $s$ of the auction, a bidder pays a flat (positive) participation cost of $c$ dollars and selects a bid $i_s \in \{b_1, b_2, \ldots, b_D\}$ along with the other competing bidders; the bidder wins the auction if their bid was the highest. If the bidders wins the auction, they get a reward of $V_s$ and consume a budget of $c + M_s$ dollars, where $M_s$ is the second highest bid. If the bidder loses the auction, their reward is $0$ but they consume $c$ dollars of their budget.
The game ends once the budget is no longer positive.
Under appropriate stochastic assumptions on the items and bids --- namely, that the items are drawn i.i.d.~(so that the utilities $V_s$ are i.i.d.) and that the highest bids $M_s$ among the other bidders also are i.i.d.\footnote{We allow $V_s$ and $M_s$ to be dependent.} --- this problem can be cast into the \fta{} framework where there is a single macro-arm and the $D$ micro-arms are the values of the bids.
\paragraph{k-fold cross-validation.}
Consider the problem of performing hyperparameter selection via $k$-fold cross-validation (CV). In $k$-fold CV, we are required to run a learning algorithm to convergence a very large number of times. For each of a typically large number of hyperparameter configurations, we need to run the learning algorithm on each of $k$ subsets of the training data. Each execution can take a different amount of time to converge, and with random initialization the runtime and also the model learned by the algorithm are stochastic. Suppose that we have a fixed time budget and view the quality of the model learned as potential reward (high quality solution meaning a large reward); then a natural goal is to find a set of hyper-parameters that are near-optimal. A natural formulation of this problem is to cast it as a pure-exploration multi-armed bandit problem ~\citep{li2017hyperband}. In this paper, we cast the $k$-fold CV problem as a regret minimization problem, inspired by a similar regret minimization approaches used in bandit convex optimization \citep{agarwal2011stochastic}. In practice, $k$-fold CV is done with an implicit time budget constraint, where long-running experiments are terminated, receiving a reward of $0$, and the experiments are re-started from a different parameter configuration. This practical consideration means that we want to maximize the total cumulative reward under a given total time budget and hence makes this problem  fit well into the \fta{} framework: the macro-arms are the various parameter configurations, and the micro-arms are the amount of time/computational resources we are willing to allocate for different experiments.

\paragraph{Computational advertising.} 
In computational advertising, a publisher wants to show an ad from an inventory of ads. When a user sees a published ad, the user might be interested in it but might not click on the ad immediately. In such cases, there is an economic incentive for the publisher to display the ad for multiple time periods and wait for a response from the user rather than switch to a different ad immediately. However, at the same time the publisher (learning agent) would like to minimize their regret of not showing the best ad. This problem can be cast in the \fta{} framework, where the macro-arms are the various ads and the micro-arms are the different durations for which the publisher is willing to wait.

\section{Related work}
\label{sec:related-work}

On the surface, \fta{} problems bear close similarity to the problem of learning the optimal waiting time. In the latter problem, studied in detail by \cite{lattimore2014learning}, in each of a fixed number of rounds a learning agent selects a waiting time. If a stochastic delay exceeds this waiting time, the agent suffers a loss equal to the waiting time plus a fixed cost;  
otherwise, the agent suffers the stochastic delay itself plus a different fixed cost.  
While learning an optimal waiting time is a common thread between the work of \cite{lattimore2014learning} and our work, there are three key differences: in an \fta{} problem, \emph{(i)} the time spent affects a budget; \emph{(ii)} there is separate collection of stochastic reward (which is not present at all in the work of \cite{lattimore2014learning}); and \emph{(iii)} the game has a random stopping time that depends on all the actions taken, making exploration more challenging. Consequently, the optimal waiting time differs considerably in our setting, instead depending on a ratio of means.

The \fta{} framework is however very related to the bandits with knapsacks (BwK) setting \citep{badanidiyuru2013bandits} (see also the earlier work of \cite{tran2012knapsack} and \cite{ding2013multi}). A BwK problem is a generalization of the classical multi-armed bandit problem \citep{lai1985asymptotically} in which the learning agent has finite quantities of a number of resources, and each pull of an arm stochastically consumes each resource while also yielding some stochastic reward (the reward and resource consumptions in each round can be dependent). The game ends once any resource is exhausted. The classical multi-armed bandit problem is recovered by taking time as the single resource, which deterministically decreases by 1 when any arm is pulled (the game ends when the finite time budget is exhausted). \fta{} problems also can be cast in the BwK setting, now by taking time as a single resource which is consumed stochastically. However, the full BwK setting is so general that the algorithms developed for this setting, and the type of regret guarantees given, differ substantially from the type of guarantees we seek here. In particular, we seek problem-dependent bounds with regret growing logarithmically with the size of the budget. Such bounds previously were obtained by \cite{ding2013multi}, \cite{xia2016budgeted} and \cite{flajolet2015logarithmic}, but, as we explain in Section~\ref{sec:regret}, in the case of \fta{} problems our results can be better in important regimes.

Finally, we mention in passing that there also are weak connections to two other settings. In bandits with delayed feedback \citep{joulani2013online}, feedback from arms can be delayed, but the learning agent can still pull an arm in every round; this difference is critical. In bandits with lock-up periods \citep{komiyama2013multi}, feedback is not delayed and an arm can be pulled in each round, but the learning agent experiences ``lock-up'' periods during which it must constantly pull the same arm. This is similar to our waiting period, but an important difference is that at the end of a waiting period in our setting, the learning agent only receives one reward (possibly equal to zero), whereas in the lock-up period setting, a reward is received in each round.

\section{A ratio estimator for \fta{} problems}
\label{sec:optimal_policy}

In this section, we find a suitable metric that captures both the reward and resource aspects of \fta{}  problems.
An upper confidence bound for this metric will be key in designing the \starucb{} algorithm in Section~\ref{sec:starucb}. For notational convenience, in this section we only consider the case of one macro-arm, so $K=k=1$ (but of course with multiple micro-arms).

Given a finite time budget and full knowledge of the 
data-generating distribution, which constant policy maximizes expected cumulative reward? As pulls can have stochastic extent, intuitively this policy should always pull the micro-arm that maximizes the ratio of expected reward to expected waiting time. Our first main result gives formal backing to this intuition.

\begin{theorem} \label{thm:exp-cumulative-reward}
Let $(V_1, \tau_1), (V_2, \tau_2), \ldots$ be i.i.d. according to a joint distribution $Q$ over $[0, 1] \times [D]$.
Consider the constant policy which pulls micro-arm $j$ in each epoch, so that in a given epoch $s$ this arm consumes (i.e.,~waits for)
$\cost{j}{s} := \min\{j, \tau_s\}$ units of time
and collects reward $\reward{j}{s} := \ind{j \geq \tau_s} \cdot V_s$. 

Then the total cumulative reward collected by this constant policy under a time budget of $T$ is equal to
\begin{align*}
\E \left[ \sum_{s=1}^L \reward{j}{s} \right]
= T \cdot \frac{\E [ \reward{j}{1} ]}{\E [ \cost{j}{1} ]} + A_j
\end{align*}
for some constant $A_j \in [0, j]$, where $L = \max \{ n: \sum_{s=1}^n \cost{j}{s} \leq T \}$ is the last epoch before the game ends. 
\NoEndMark
\end{theorem}

From the above theorem, it is clear that to maximize expected cumulative reward, we should devise an estimator for the expected per-round reward of each micro-arm.
Let us introduce notation for the quantities we wish to estimate. In the following, let $(V, \tau) \sim Q_1$. 
For $j \in [D]$, define the expected per-round reward
\begin{align}
\gtrue{1}{j} := \frac{\E \left[ \ind{j \geq \tau} \cdot V \right]}{\E \left[ \min\{j, \tau\} \right]} . \label{eqn:g}
\end{align}
A natural estimator of $\gtrue{1}{j}$ is
\begin{align}
\ghat{1}{j}{s}
:= \frac{\sum_{m=1}^s \ind{i_m= (1,j)} \cdot \ind{j \geq \tau_m} \cdot V_m}
           {\sum_{m=1}^s \ind{i_m = (1,j)} \cdot \min\{j, \tau_m\}} . \label{eqn:one-g-estimate}
\end{align}
The above expression records the total reward received from pulls of arm-pair $(1,j)$ divided by the total rounds spent during these pulls.
Comparing to Theorem~\ref{thm:exp-cumulative-reward}, we can see that the numerator of \eqref{eqn:one-g-estimate} is an unbiased estimator of the numerator of Theorem~\ref{thm:exp-cumulative-reward}, and the same relationship holds between the respective denominators.
However, the full estimator above is not an unbiased estimator of the expected per-round reward $\gtrue{1}{j}$, as is readily observed from Jensen's inequality. It is easy to see that \eqref{eqn:one-g-estimate} can easily be generalized to the case of multiple macro-arms as follows:
\begin{align}
\hspace{-1mm} \ghat{k}{j}{s} \hspace{-0.5mm} 
:= \hspace{-0.5mm} \frac{\sum_{m=1}^s \ind{i_m= (k,j)} \cdot \ind{j \geq \tau_{k,m}} \cdot V_m}
           {\sum_{m=1}^s \ind{i_m = (k,j)} \cdot \min\{j, \tau_{k,m}\}} .  \label{eqn:g-estimate}
\end{align}

\section{\starucb{}}
\label{sec:starucb}
Our approach to solving \fta{} problems is to develop an upper confidence bound-style algorithm based on the reward per round estimate from \eqref{eqn:g-estimate}. To achieve this, we develop a concentration inequality for how much higher the mean expected per-round reward may exceed this estimate. The following lemma gives us the concentration inequality for our quantity of interest.

\begin{lemma} \label{lemma:ratio-improved}
Let $X, Y$ be (possibly dependent) random variables with joint distribution $P$.
Consider a sample $(X_1, Y_1), \ldots, (X_n, Y_n)$ of independent copies of $(X, Y) \sim P$. Assume that $X$ takes values in $[0, 1]$ and $Y$ takes values in $[1, B]$. Define $\mu_Y := \E[Y]$ and let $\hat{X}$ denote the sample mean of $X_1, \ldots, X_n$ (likewise for $\hat{Y}$ and $Y_1, \ldots, Y_n$). For any choice of $\delta \in [0, 1]$, we have with probability at least $1 - \delta$ over the sample,
\begin{align*}
\textstyle
\left| \frac{\hat{X}}{\hat{Y}} - \frac{\E[ X ]}{\E[ Y ]} \right|
\leq \sqrt{\frac{(B - 1) \log \frac{4}{\delta}}{2 n}} 
          + \frac{2(B-1) \log \frac{4}{\delta}}{3 \mu_Y n}
          + \sqrt{\frac{\log\frac{4}{\delta}}{2 n} } .
\end{align*}
\NoEndMark
\end{lemma}
A proof of this lemma can be found in Appendix~\appref{app:proof-of-ratio-improved-lemma}.

With Lemmas~\ref{lemma:ratio-improved} and \eqref{eqn:g-estimate} in hand, we now introduce the new algorithm, Algorithm~\ref{gucb}. Since the arms correspond to waiting times, we call this algorithm ``\starucb{}''.\footnote{Also, ``wait'' is a homophone for ``weight'': the upper deviation terms are weighted by the waiting-time-dependent quantities $\alpha_j$ and $\beta_j$ introduced below.} 
From Lemma~\ref{lemma:ratio-improved} and our later results in Section~\ref{sec:pre-regret-bound}, for any arm pair $(k,j)$ and non-negative integers $s$ and $n$, the upper deviation around $\ghat{k}{j}{s}$ is given by
\begin{align*}
\updev{k}{j}{s} = \alpha_j \frac{\log s}{\npulls{k,j}{s}} + \beta_j \sqrt{\frac{\log s}{\npulls{k,j}{s}}} ,
\end{align*}
where $\alpha_j \defeq \frac{8(j-1)}{3}$ and $\beta_j\defeq \sqrt{2}(\sqrt{j-1}+1)$ are constants independent of the macro-arms and $\npulls{k,j}{s}$ is the number of pulls of arm pair $(k,j)$ at the end of epoch $s$.

\begin{algorithm}\label{algorithm}
\caption{\starucb{} Algorithm}\label{gucb}
\SetAlgoLined
\DontPrintSemicolon
\textbf{Input:} time budget $T$, maximum waiting time $D$ \;
$\alpha_j \leftarrow \frac{8(j-1)}{3},  \beta_j \leftarrow \sqrt{2}(\sqrt{j-1}+1)$ for all $j \in [D]$\;
$\npulls{k,j}{0} = 0$ for all $(k,j) \in [K] \times [D]$ \;
$s = 1$ \;
\While{$T > 0$}{
Pull $i_s \gets$ $\argmax\limits_{(k,j) \in [K] \times [D]}   { \ghat{k}{j}{s-1} + \updev{k}{j}{s-1} }$\tcp*{see \eqref{eqn:g-estimate} for $\ghat{k}{j}{s-1}$}
$T \gets T - \min \{j_s,\tau_s \}$, where $i_s = (k_s, j_s)$\;
$\npulls{(k,j)}{s} \gets \npulls{(k,j)}{s-1} + \ind{i_s = (k,j)} \,\, \forall (k,j)$\;
$s \gets s + 1$\;
}
\end{algorithm}
Let us summarize the main idea behind the algorithm. Before an arm is pulled at least once, all the upper confidence bounds are initialized to infinity. This forces us to explore each arm at least once. After this phase, the arm that has the highest upper confidence bound is played. As an arm is played often, from the concentration inequality given in Lemma~\ref{lemma:ratio-improved}, our estimate of the ratio of reward to resource consumed by the arm becomes sharper and we converge to the optimal arm.

\section{Expected regret of \starucb{}}
\label{sec:regret}

We begin bounding the regret of \starucb{} by bounding the number of times a suboptimal arm is pulled.
\subsection{ Bound on expected number of pulls}
\label{sec:pre-regret-bound}

\begin{lemma}\label{lemma:exp-sub-pulls}
For any sub-optimal arm pair $(k,j)$: $\gap{k}{j} \defeq \gtrue{k^*}{j^*} - \gtrue{k}{j} > 0$, the expected number of pulls in $L$ epochs is given by
\begin{align*}
\E[\npulls{k,j}{L}] \leq  \log(T)\left(\frac{\beta_j + \sqrt{\beta_j^2+2\alpha_j\gap{k}{j}}}{\gap{k}{j}}\right)^2+ \frac{4 \pi^2}{3} .
\end{align*}
\NoEndMark
\end{lemma}
A proof of this lemma can be found in Appendix~\appref{app:proof-of-exp-sub-pulls-lemma}.

\begin{lemma} \label{lemma:suff-sample}
Let $l = \log(T)\left(\frac{\beta_j + \sqrt{\beta_j^2 + 2 \alpha_j \gap{k}{j}}}{\gap{k}{j}} \right)^2$. Then for any epoch $s$ such that $l \leq s \leq L $,
\begin{align*}
\Pr \left( \left| \gtrue{k}{j} - \ghat{k}{j}{s} \right| \geq \epsilon \right) \leq 4 s^{-4} ,
\end{align*}
where $\epsilon = \alpha_j \frac{\log s}{\npulls{k,j}{s}} + \beta_j \sqrt{\frac{\log s}{\npulls{k,j}{s}}}$.
\NoEndMark
\end{lemma}
The proof of the above lemma (see Appendix~\appref{app:proof-of-suff-sample-lemma}) 
is based on tuning the values of $\alpha_j$ and $\beta_j$ in the algorithm. It follows by a heavy sequence of algebraic steps and is omitted from the main text for brevity. 

\subsection{Regret bound}
\label{sec:regret-bound}

\begin{theorem}\label{regret}
\starucb{}'s pseudo-regret is at most
{\small
\begin{align*}
\sum_{(k,j) | \gap{k}{j} > 0} \mu_{k,j}^{(c)} \left( \frac{\left( \beta_j+ \sqrt{\beta_j^2  + 2 \gap{k}{j} \alpha_j } \right)^2 \log T}{\gap{k}{j}}  + \bo(1) \right) ,
\end{align*}
}
where  
we define the mean waiting time $\mu_{k,j}^{(c)} := \E [ \cost{k,j}{1} ]$.
\NoEndMark
\end{theorem}

The $\bo(1)$ only hides moderate constants and scales as $\Delta_{k,j} \leq 1$. 
The leading term (involving $\log T$) in the above bound is less explicit due to the notation; a coarser version of the leading term is
\begin{align}
\bo \left( \sum_{(k,j) | \gap{k}{j} > 0}
\mu_{k,j}^{(c)} \left( \frac{j \log T}{\gap{k}{j}} \right) \right) \label{eqn:regret-simplified} .
\end{align}
In comparing this result to previous regret bounds of \cite{flajolet2015logarithmic},
 \cite{xia2016budgeted}, and \cite{ding2013multi} 
for UCB-Simplex, Budget-UCB, and UCB-BV1 respectively, we focus on the case of a single macro-arm ($K=1$), as this is enough to capture the ``waiting'' aspect of the problem. Let us assume that all suboptimal arms have gap lower bounded by $\Delta > 0$. Then since all the $j$ are at most $D$, the leading term \eqref{eqn:regret-simplified} in our regret bound is of order at most
\begin{align}
\left( \frac{D^2}{\Delta} \bar{\mu}^{(c)} \right) \log T , \label{eqn:WAIT} \tag{WAIT}
\end{align}
where we define $\bar{\mu}^{(c)} := \frac{1}{D} \sum_j \mu_{1,j}^{(c)}$. 
Now, in the worst case  (when mean waiting times are high), this term becomes $(D^3 / \Delta) \log T$. However, for easier problems where the mean waiting time for most arms is much smaller than $D$, the term improves to $(D^2 / \Delta) \log T$.  
Before comparing this result to the regret bounds of \cite{flajolet2015logarithmic}, \cite{xia2016budgeted}, and  \cite{ding2013multi}, it is important to note that those works have stochastic resource consumptions lying in $[0,1]$. Therefore, to view the \fta{} framework in their setting, we rescale our consumptions from $[D]$ to $\{\frac{1}{D}, \frac{2}{D}, \ldots, 1\}$. Consequently, for each gap $\Delta_{1,j}$ in our paper, the corresponding gap in their paper will be scaled up by $D$. 

With this conversion in mind, we turn our attention to Corollary 1 of \cite{ding2013multi}. After some unpacking, one can see that in an \fta{} problem, the leading term of their regret bound is of order 
\begin{align}
\left( D^4 + \frac{D^3}{\Delta} + \frac{D^2}{\Delta^2} \right) \frac{\mu_{1,j^*}^{(r)}}{\mu_{1,j^*}^{(c)}} \log T. \label{eqn:ding} \tag{DQZL}
\end{align}
In the situation where the optimal arm's mean waiting time $\mu_{1,j^*}^{(c)} \in [1, D]$ is small, their bound is noticeably worse: the first term is quartic in $D$, the second term matches our worst-case bound, and the third term  is quadratic in $(1/\Delta)$. Yet, when the mean waiting time for the optimal arm is large, their bound becomes closer to the behavior of our bound were $\bar{\mu}^{(c)}$ to be small. In either case, their bound grows as $\frac{1}{\Delta^2}$. 

Next, we compare our regret bound to a regret bound of \cite{flajolet2015logarithmic} for UCB-Simplex. After converting their notation to ours, their Theorem 1 gives regret that is of order
\begin{align}
\left( \frac{D^2}{\Delta} \sum_j \frac{1}{\mu_{1,j}^{(c)}} \right) \log T  
 + \bo(D^3) . \label{eqn:FJ} \tag{FJ}
\end{align}
The leading terms in \eqref{eqn:FJ} versus \eqref{eqn:WAIT} are close but there is a important distinction. Both leading terms contain a common factor of $\frac{D^2}{\Delta}$. However, \eqref{eqn:WAIT} has the average of the mean waiting times, while \eqref{eqn:FJ} has the sum of the reciprocal mean waiting times. When the mean waiting times are small, the average is smaller. When the mean waiting times are larger, the sum of reciprocals is smaller. Each quantity has a range of $[1, D]$. However, in any case, the constant term in \eqref{eqn:FJ} can be of order $D^3$, whereas the constant term from Theorem~\ref{regret} (not shown but visible in the proof) is \emph{always} $\bo(D^2)$. A similar comparison can also be made for Budget-UCB from Theorem 3 of \cite{xia2016budgeted} which gives the same logarithmic leading term as in \eqref{eqn:FJ} when the budget is sufficiently large.

We posit that the reason for our regret bound's improvement over the other bounds (in some regimes) is that our analysis is quite different: we directly form an upper confidence bound for the ratio estimator, and this gives us an opportunity to leverage a Bernstein-style improvement (from Bernstein's inequality). It would be interesting to somehow combine the style of analysis used in this work and the style from one of the other works to get a bound that dominates in all regimes. However, based on the experimental results in Section~\ref{sec:expt}, it might be that our regret bounds for the \emph{existing} \starucb{} algorithm could be improved.

\section{Experiments}
\label{sec:expt}

Our experiments focus on testing \starucb{} in three different scenarios in the \fta{} framework.
In all our experiments, we compare \starucb{} to UCB-Simplex \citep{flajolet2015logarithmic}  and Budget-UCB \citep{xia2016budgeted}; recall that these algorithms also admit  logarithmic regret bounds in the BwK setting. Also, UCB-BV1 \citep{ding2013multi} was considered in all experiments, but it was later excluded
as the other algorithms performed
better. The delay $\tau_k$ and potential reward $V_k$ are chosen such that they have a moderate minimum gap $\Delta$. The delay  distribution   $\tau_k$ over $[D]$ for each experiment is given as a bar graph above the cumulative regret figures. Each experiment in this section is an average over 10 independent runs, and we use the same time budget of $ T=10^7$ rounds for all the experiments. The pseudo-regret in each experiment is calculated as
\begin{align*}
R_t = t \cdot \gtrue{k^*}{j^*} -  \sum_{s=1}^{L_t} \reward{i_s}{s} ,
\end{align*} 
where $L_t$ is the algorithm's last epoch for budget $t$.
\subsection{One macro-arm and several micro-arms}
We start by having just 1 macro arm with deterministic potential reward $V_{k,s}=1$ and $D=10$ micro-arms  configured with $\tau_k$ such that the optimal arm falls in 
different intervals of $[D]$,
allowing us to understand the algorithms' behavior for various $\tau_k $. 

In the first experiment, we chose $\tau_k$ such that the delay doubles with the rewards obtained (see the bar graph in Figure \ref{fig:doubling}). For instance, if an algorithm chooses to wait for only 1 round, it consumes less resources and so can play for more epochs; on the flip side, if it decides to wait slightly longer than two rounds, it has twice the chance of getting the reward. This delay distribution induces a minimal gap of $\Delta = 0.042$.  
The learning algorithm must navigate a tight trade-off to select the optimal waiting time.
In the next experiment, we chose the delay $\tau_k$ such that the optimal arm lies in the middle of $[D]$ incurring a moderate gap of $\Delta = 0.124$. The final experiment in this section stems from the observation that $\alpha_j$ becomes zero for arm $j=1$, and so to study the behavior of the algorithm when the optimal arm is $1$, an appropriate delay $\tau_k$ that incurs a moderate gap of $\Delta = 0.166$ is chosen. 

Figure~\ref{exp_10_micro_arm} shows the performance of \starucb{} for the configurations of delay distribution that were discussed before. These experiments demonstrate a few interesting insights. From Figures  \ref{fig:doubling} and \ref{fig:mid_best}, we can see that for the first two experiments, 
 \starucb{} performs much better than UCB Simplex and Budget-UCB. We believe that the main reason for this comes from the underlying principle on which UCB algorithms are developed, i.e., Optimism in the Face of Uncertainty. In our \fta{} framework,  the uncertainty in getting a reward decreases if we choose to  wait for longer time. This behavior is very well captured by \starucb{} in the construction of the confidence radius (which grows with $j$), resulting in relatively quick convergence towards the best arm. Also, observe that $\alpha_j$ becomes $0$ for arm $j=1$, and in the last experiment (Figure \ref{fig:one_best}), \starucb{} performs worse than the other algorithms; however, it still enjoys logarithmic regret.

\begin{figure*}[h]
    \begin{tabular}{ccc}
        \centering
        \begin{subfigure}[b]{0.27\textwidth}
            \includegraphics[width = \textwidth]{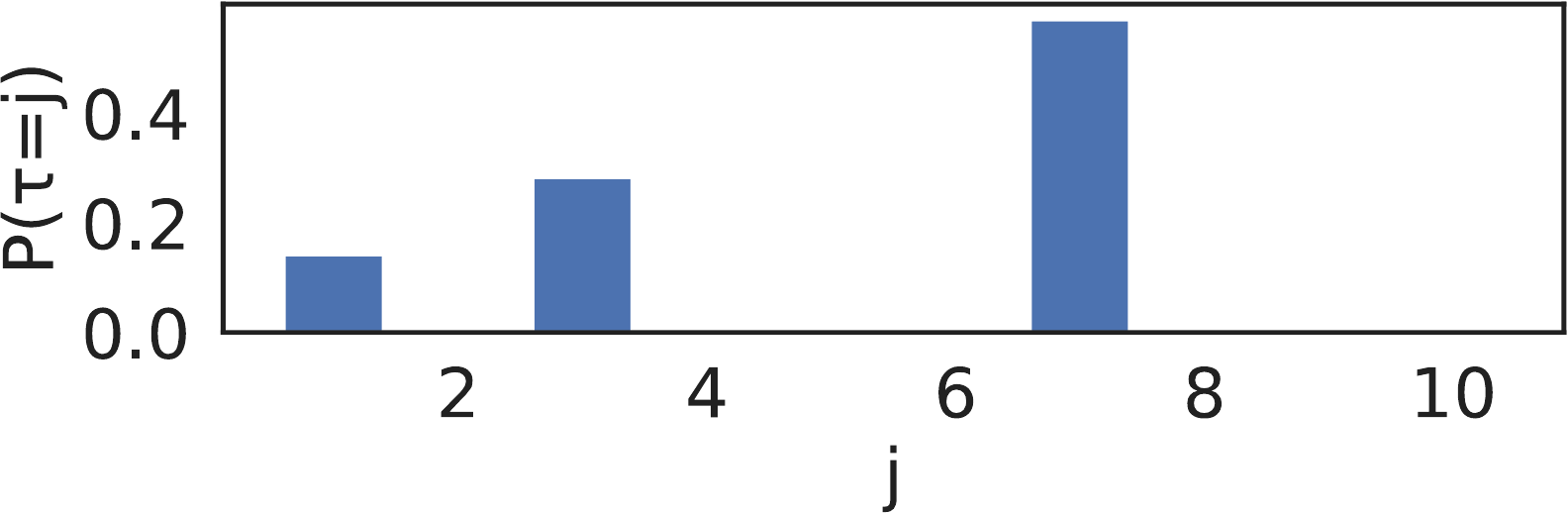}
        \end{subfigure}&
        \begin{subfigure}[b]{0.27\textwidth}
            \includegraphics[width=\textwidth]{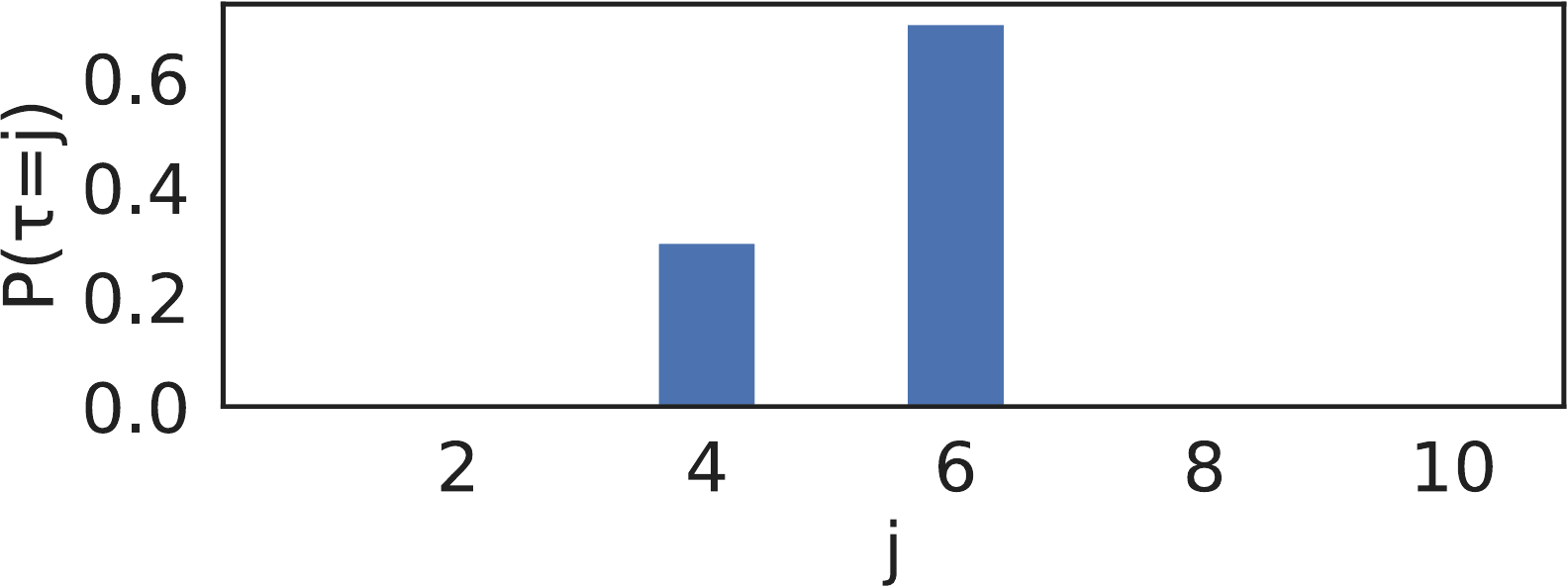}
        \end{subfigure}&

        \begin{subfigure}[b]{0.27\textwidth}
            \includegraphics[width=\textwidth]{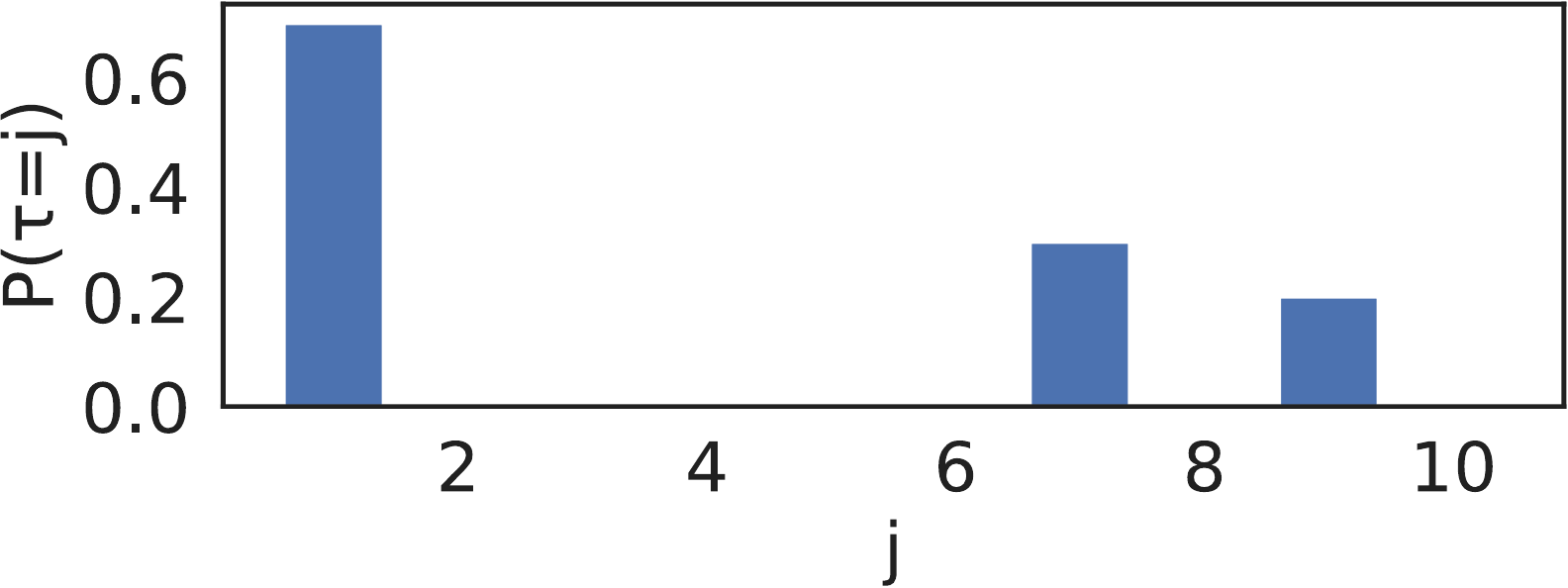}
        \end{subfigure}
    \\
        
        \begin{subfigure}[b]{0.3\textwidth}
            \includegraphics[scale=0.35]{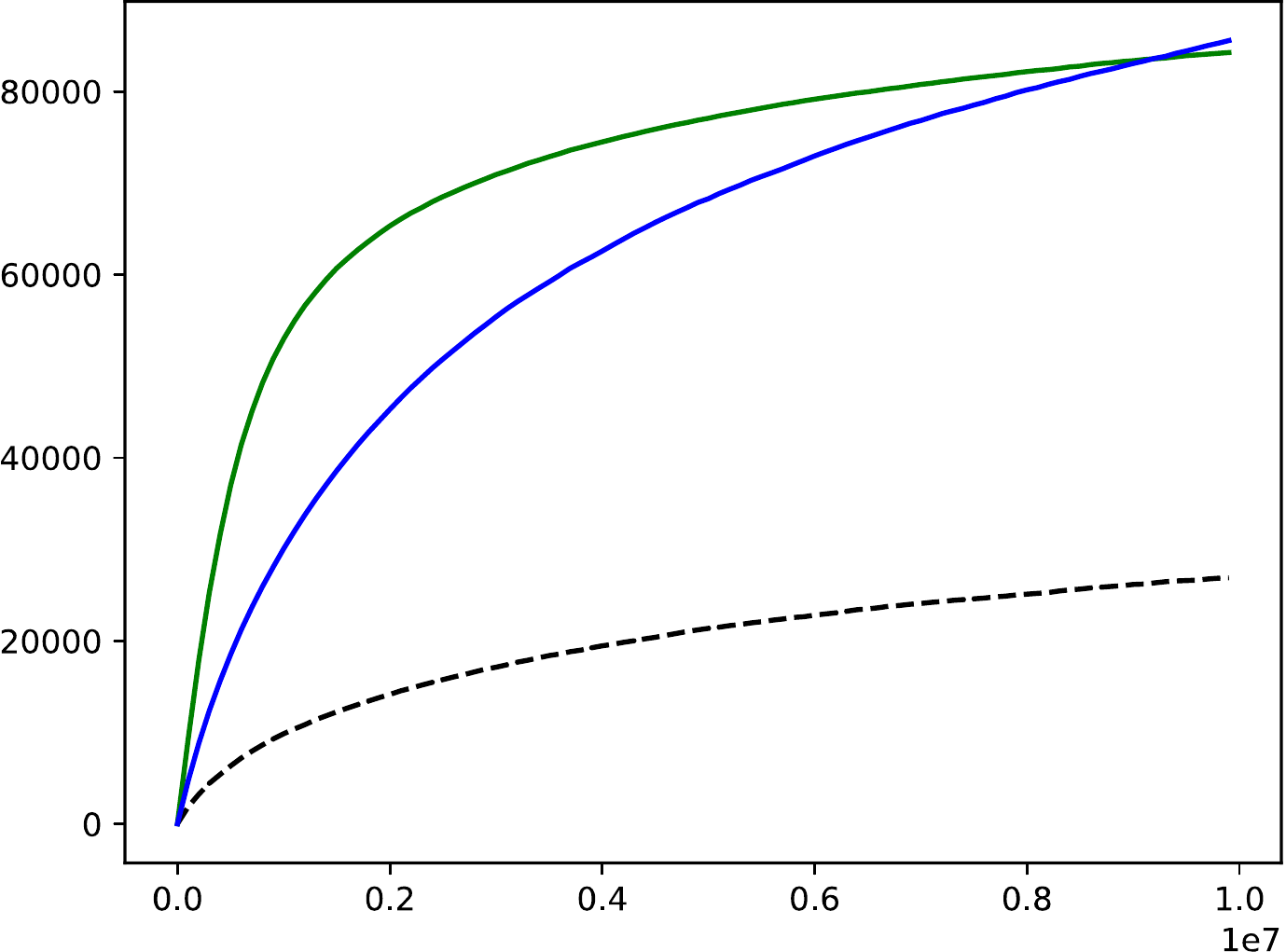}
                    \caption{Doubling Scenario: $j^* \gtrsim [D/2]$}        
            \label{fig:doubling}
        \end{subfigure}&

        \begin{subfigure}[b]{0.3\textwidth}
            \includegraphics[scale=0.35]{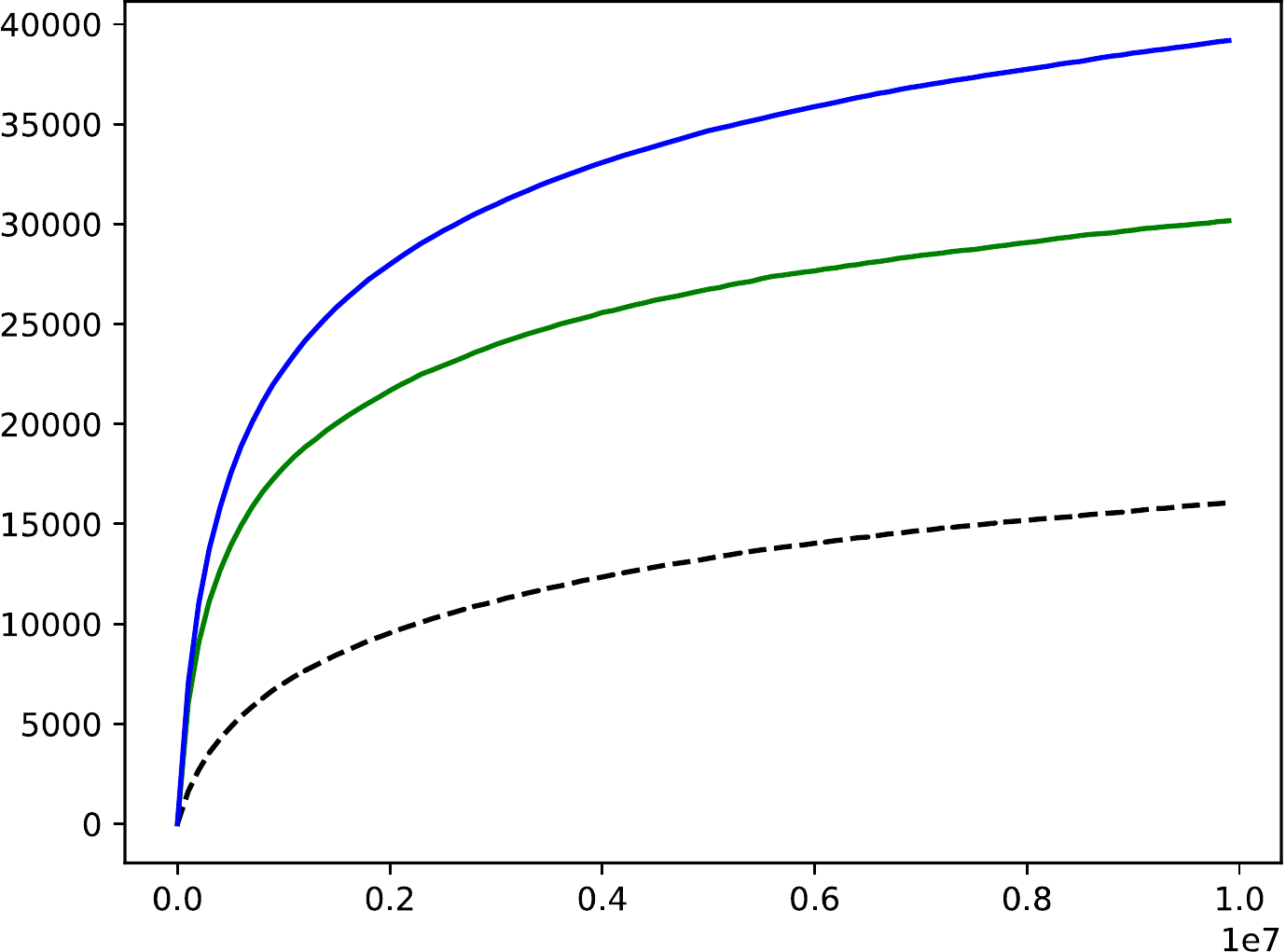}
            \caption{$j^* \approx [D/2]$ }
            \label{fig:mid_best}
        \end{subfigure}&

        \begin{subfigure}[b]{0.3\textwidth}
            \includegraphics[scale=0.35]{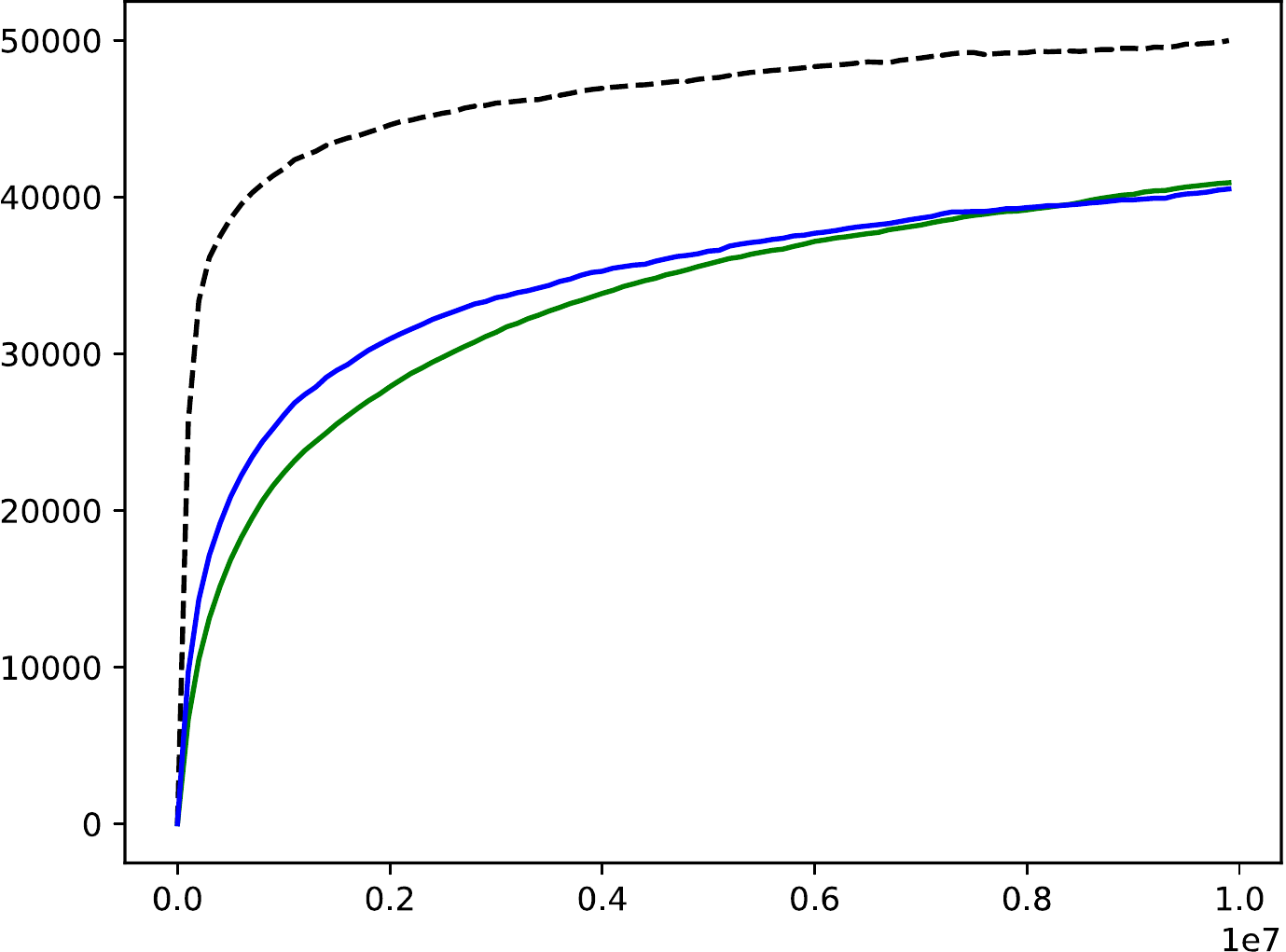}
            \caption{ $j^* = 1$}
            \label{fig:one_best}
        \end{subfigure}
    \end{tabular}
    \caption{Cumulative Regret of \starucb{}  for $D  = 10, K  = 1, V_k = 1$}
    \label{exp_10_micro_arm}
    \end{figure*}
    
    \begin{figure*}[h]
    \begin{tabular}{c|c|}
    \cline{2-2}
    
    \begin{tabular}{c}
        \begin{subfigure}[b]
    {0.3\textwidth}
            \centering
            \includegraphics[scale=0.5]{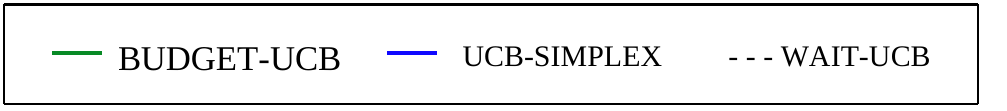}

        \end{subfigure}
        \\[0.6cm]
        \begin{subfigure}[b]{0.3\textwidth}
            \includegraphics[scale=0.35]{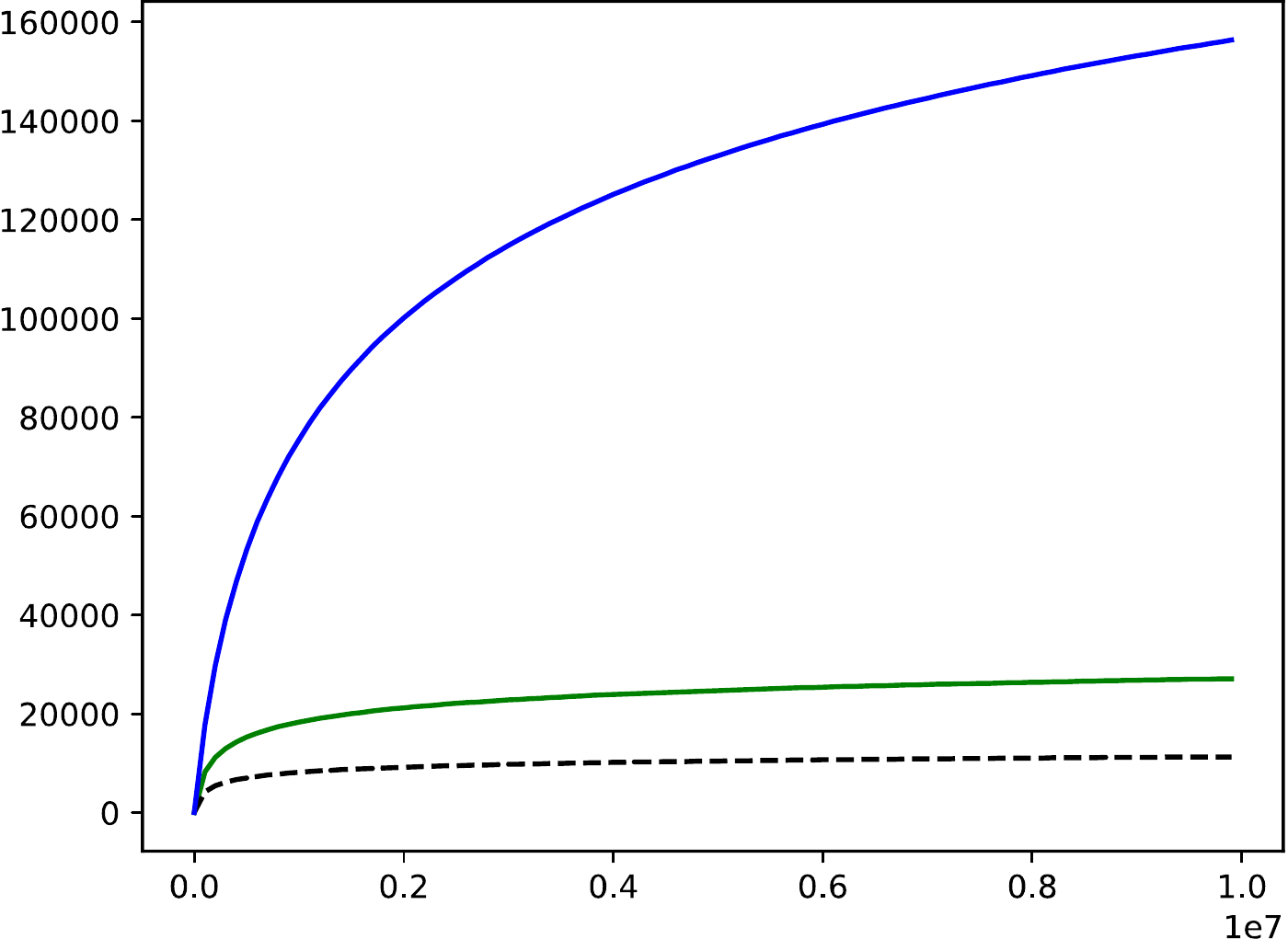}
            \caption{for $D  = 10, K  = 1, V_k = 1$}
            \label{fig:exp_standard_mab}    
        \end{subfigure}
    
    \end{tabular}  
    
    &
    
    \begin{tabular}{c}
    
    \begin{tabular}{@{\hskip-3pt}ccc@{\hskip-3pt}}
    
        \begin{subfigure}[b]{0.17\textwidth}
            \includegraphics[width = \textwidth]{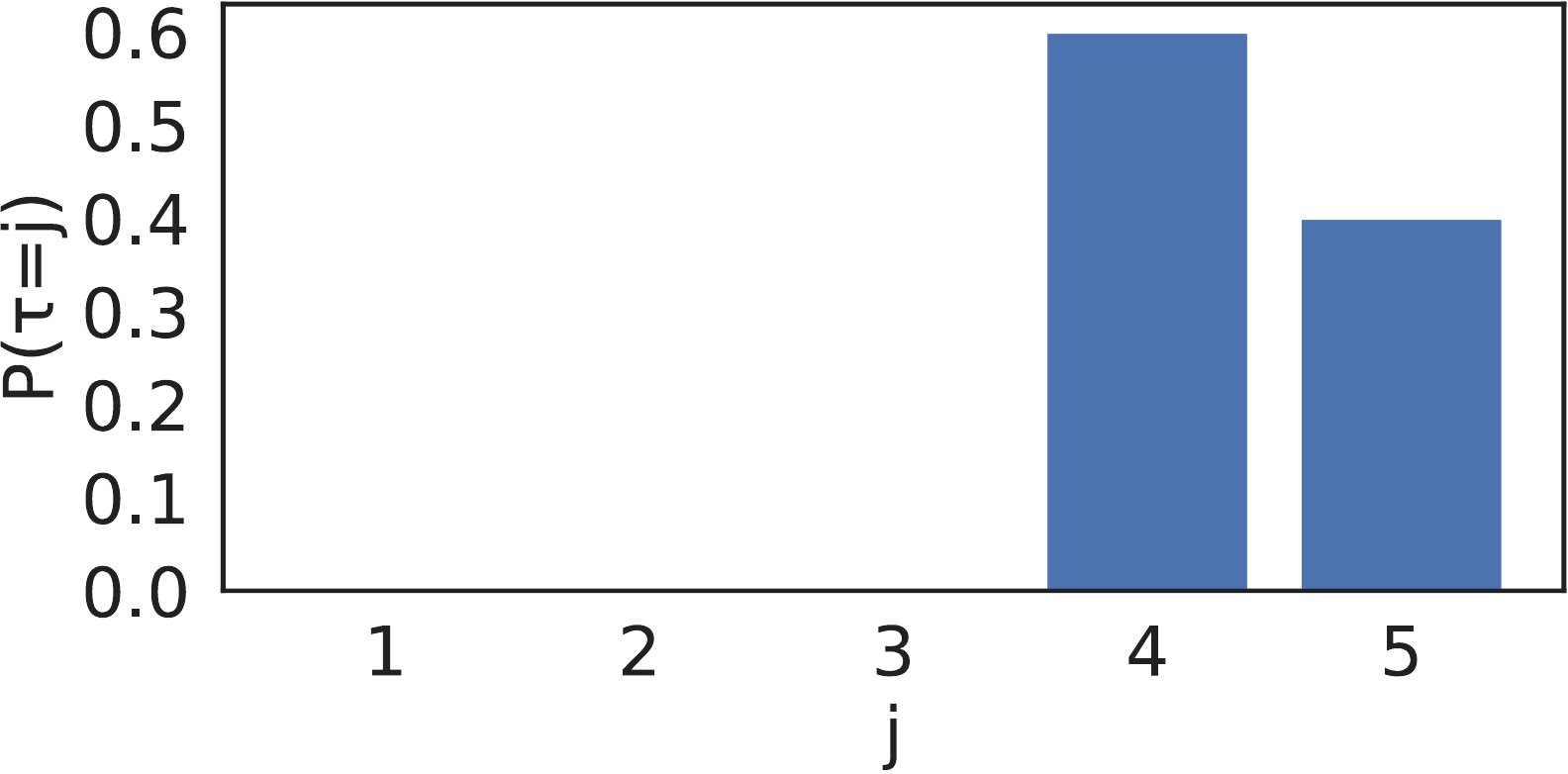}
        \end{subfigure}&

        \begin{subfigure}[b]{0.17\textwidth}
            \includegraphics[width=\textwidth]{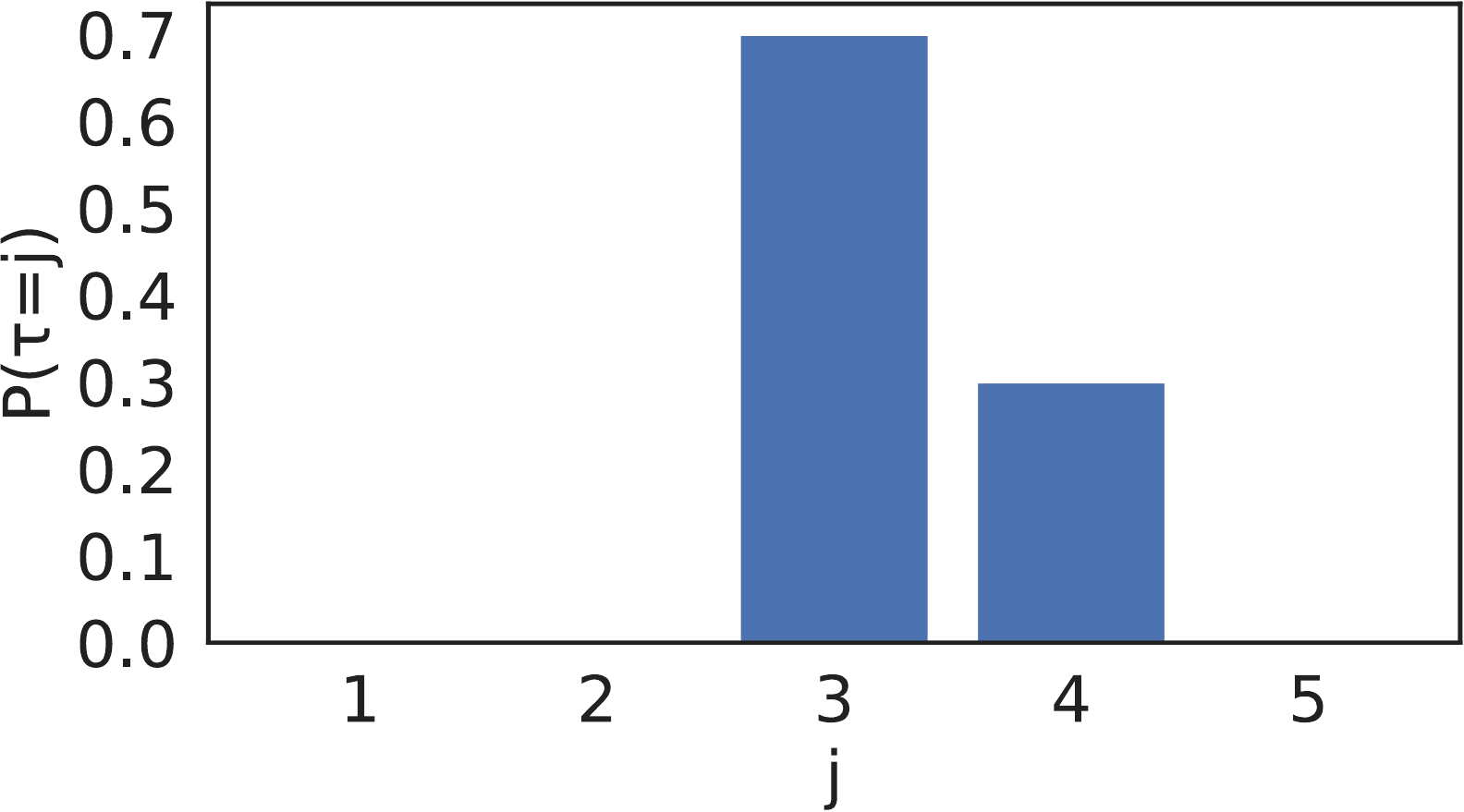}
        \end{subfigure}&

        \begin{subfigure}[b]{0.17\textwidth}
            \includegraphics[width=\textwidth]{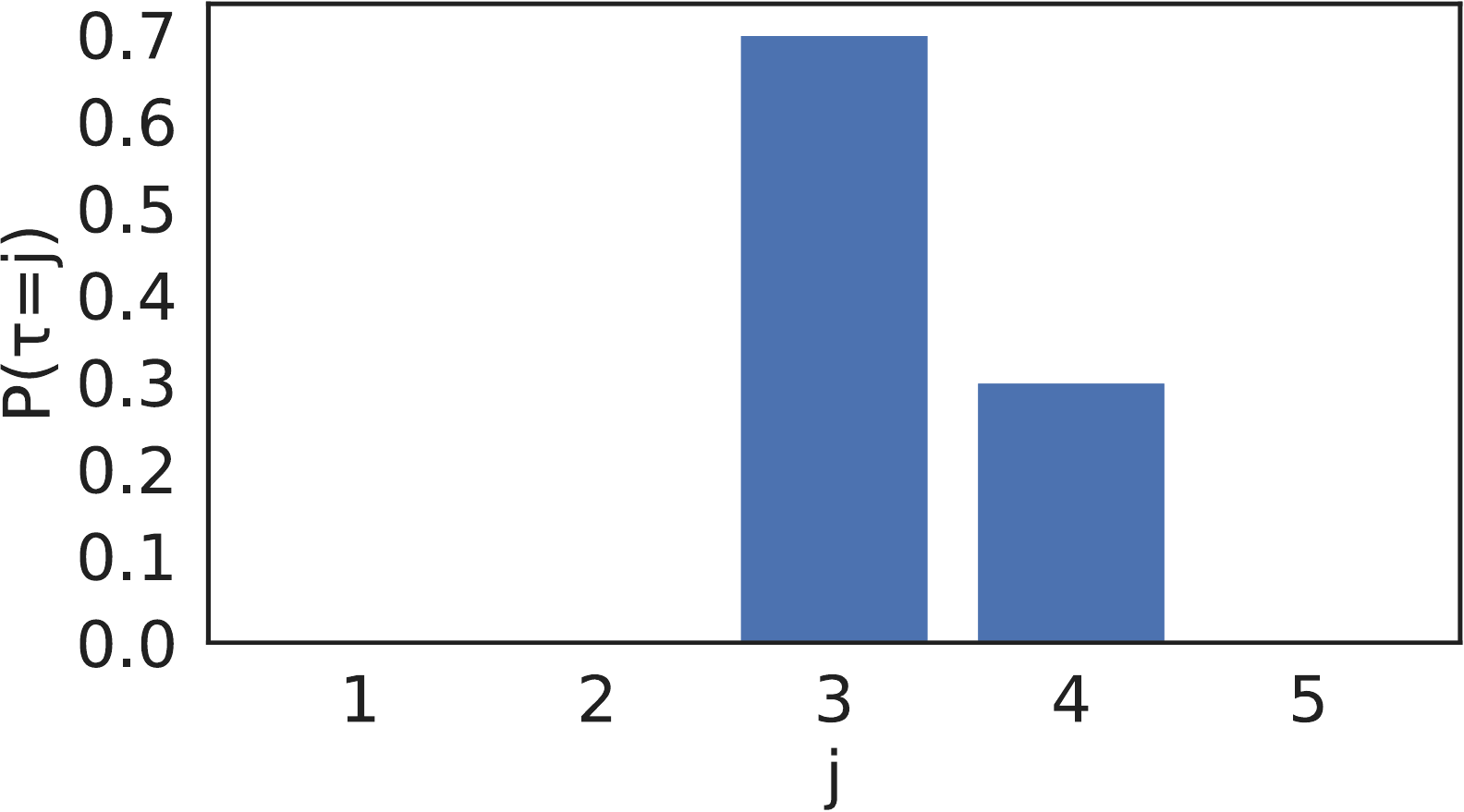}
        \end{subfigure}
    \end{tabular}
    \\
    
    \begin{tabular}{cc}
        \begin{subfigure}[b]{0.25\textwidth}
            \includegraphics[scale=0.25]{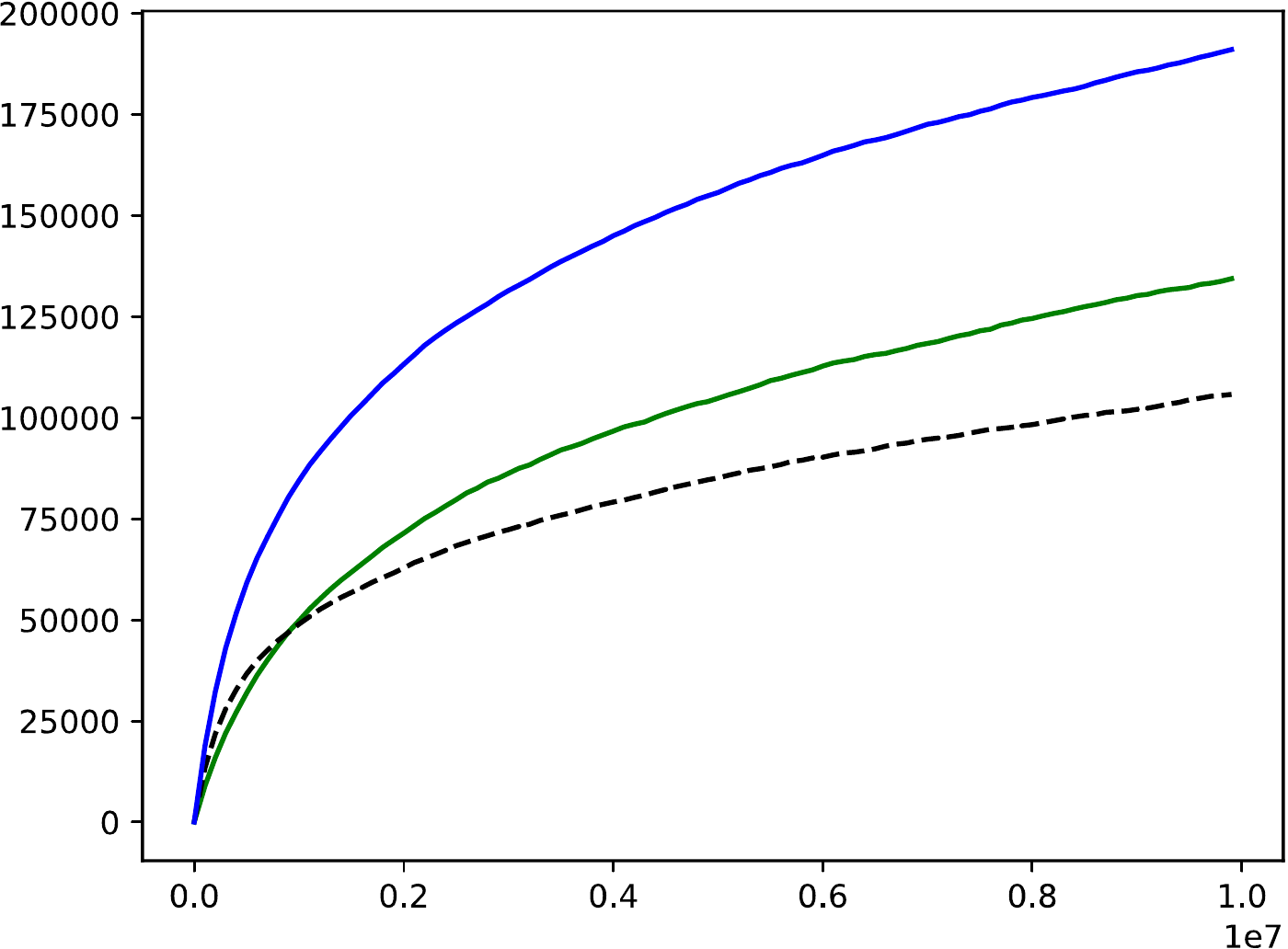}
    \caption{Case I: $D=5,K=3$}
    \label{fig:adswitch_case1}
        \end{subfigure}
         &          
         \begin{subfigure}[b]{0.25\textwidth}
            \includegraphics[scale=0.25]{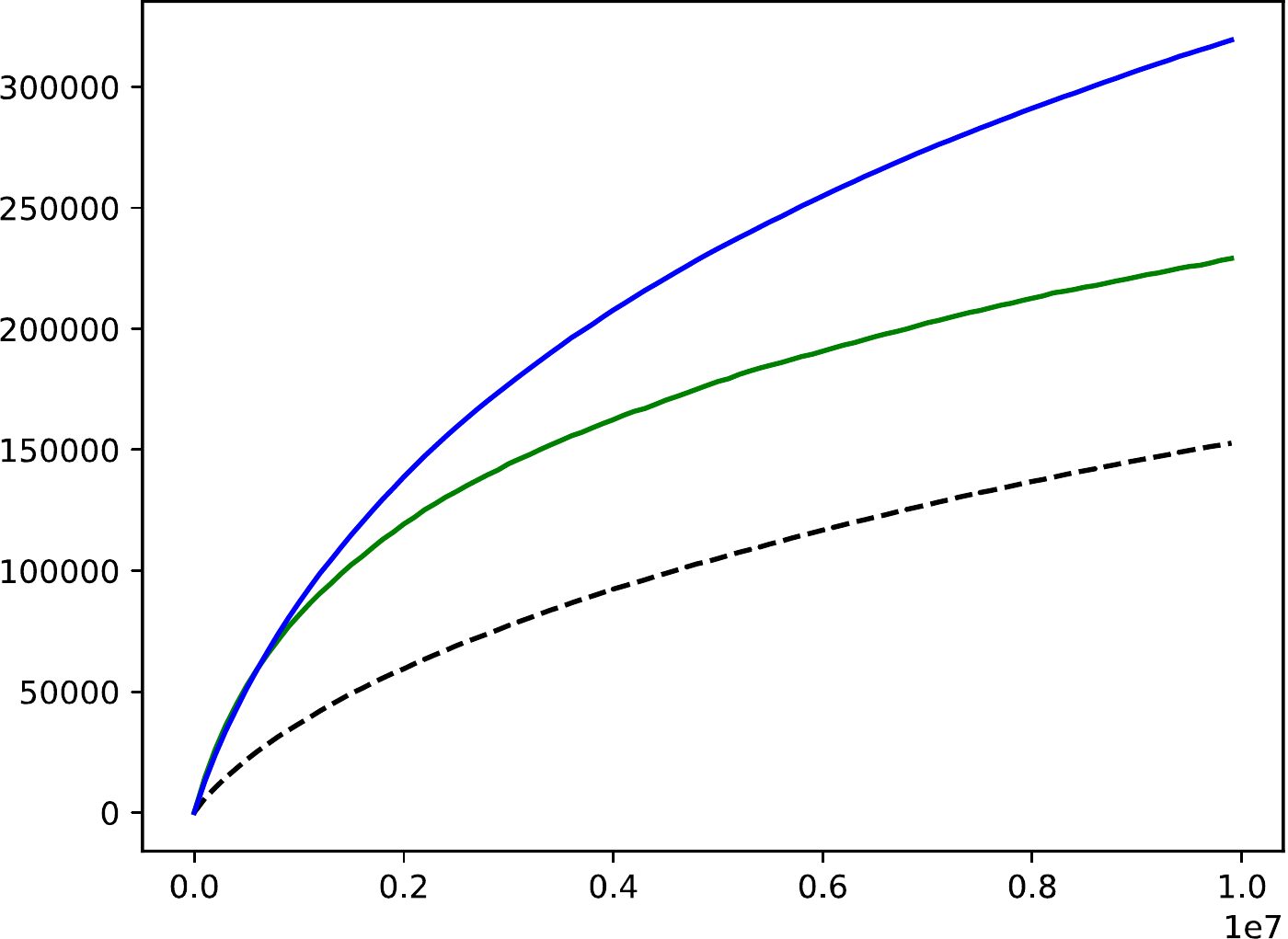}
    \caption{Case II: $D=5,K=3$}    
    \label{fig:adswitch_case2} \end{subfigure}

    \end{tabular}

    \end{tabular}
    \\
    \cline{2-2}

    \end{tabular}
    \caption{Cumulative Regret of \starucb{}}
    \end{figure*}

\subsection{Several macro-arms and one micro-arm}
By having only one micro-arm ($D=1$) and several macro-arms ($K > 1$), our problem reduces exactly to the standard MAB setting. Now, the algorithms simply need to pull the arm with highest $V_k$. A similar observation can be found in the setting when $D\geq 1$ and the delay distribution happens to be $P(\tau = 1) = 1$, i.e., any arm $j \in [D]$ will be the optimal arm as the epochs are completed in unit time (so the waiting time becomes insignificant here). Now, observe that the above two setups appear to be the same but are different in the perspective of  algorithms which compute confidence radii for all the $(k,j)$ pairs and select one among them. We decided to test the algorithms in the latter scenario with the following experimental setup. We took $K=3$ and $D=5$ with deterministic $\tau_k=1$ and $V_k$ being a Bernoulli random variable with success probabilities $0.5$, $0.7$, and $1$ respectively.
 
Note that for this case, our upper confidence bound cleanly reduces to the upper confidence bound of UCB1  \citep{Auer2002} for the standard stochastic multi-armed bandit problem, which is $a_{k,1}(s) = \sqrt{\frac{2 \log s}{N_{k,1}(s)}}$. Figure \ref{fig:exp_standard_mab} shows the results for this setting. 
It is evident from the figure that \starucb{} performs much better than  UCB-Simplex and Budget-UCB. 
 
We think this might be mainly due to the fact that the exploration term in \starucb{} scales at most by the number of micro-arms $j$ whereas, for UCB Simplex, the exploration term scales with the total number of Macro-Micro arm pairs $k \cdot D$. Due to this large exploratory factor, it takes longer for  UCB-Simplex to converge.
\subsection{Several macro- and micro-arms}

This setting is used in several of the real world applications discussed in Section~\ref{sec:application}. We performed a synthetic experiment inspired from  computational advertising. We have a set of ads belonging to $K = 3$ categories with $V_k$   their corresponding private  utility for an impression (click) and $D=5$ denotes the maximum waiting time for the ad before switching it. The delay distribution $\tau_k$ captures the time of impression of a user towards category $k$, which in a way  reflects their interest on that category $k$. Note that showing a relevant ad affects the time of impression, thereby establishing an implicit connection between reward and delay. Now, the goal of the learning algorithm is to show the ad that best fits the interest of the user and learn the optimal waiting time for their impression. 
 
We consider two cases for our experiments. In Case I, we set $V_k$ to be a Bernoulli random variable with success probabilities $0.7$, $1$, and $0.5$ respectively. In Case II, we keep the same delay $\tau_k$ which represent the same user's interest but only change the $V_k$ i.e., the publisher now receives different private utilities for different ad categories. 
The new Bernoulli random variable with success probabilities for the same categories are $1$, $0.5$, and $0.7$ respectively.  
Figures~\ref{fig:adswitch_case1} and \ref{fig:adswitch_case2} shows the learning behavior of \starucb{} in these scenarios and the bar graphs on top represent the $\tau_k$ distribution for each of the 3 categories.  

\section{Conclusion and Future work}
\label{sec:conclusion}

In this work, we introduced the Farewell to Arms framework of multi-armed bandit problems and presented a new algorithm, \starucb{}, along with a logarithmic problem-dependent regret bound of order $O((D^3/\Delta) \log T)$. In the case of a single macro-arm and when most micro-arms have low mean waiting times (much smaller than $D$), our regret bound improves to $O((D^2/\Delta) \log T)$; in this regime, the leading ($\log T$) term of our bound for \starucb{} is smaller by a factor of $D$ than the bounds of \cite{ding2013multi}, \cite{xia2016budgeted}, and \cite{flajolet2015logarithmic} 
for UCB-BV1, Budget-UCB, and UCB-Simplex respectively. However, in the opposite regime, the leading term (but not the constant term) of the bound for UCB-Simplex and Budget-UCB can be smaller than ours. Yet, our experiments show that \starucb{} empirically outperforms UCB-Simplex and Budget-UCB in a number of settings. Notably, in the standard bandit setting with only one arm, our algorithm collapses to be equal to a standard UCB algorithm, while this is not the case for the other algorithms. Indeed, our algorithm well-outperforms UCB-Simplex and Budget-UCB even in this simple setting.

In closing, we mention a few exciting directions for future work. First, it would be interesting (but challenging) to take into account that an \fta{} game potentially can have more feedback than mere bandit feedback. For example, if a learning agent decides to wait for 5 rounds, the learning agent also gains feedback for any shorter waiting time. This feedback could readily be used by \starucb{}, but proving improved regret bounds in light of this feedback is highly challenging. A second, further improvement would be to completely utilize the structure of the \fta{} game, which also includes conditional feedback. For instance, observe that if the learning algorithm decides to wait for $j \geq \tau$ rounds (for delay $\tau$), then the algorithm knows $\tau$ and hence also gets feedback for all the longer waiting times $j+1, \ldots, D$. Yet, if it does not wait for long enough (i.e., $j < \tau$), then this feedback will not be available.

\clearpage

\subsection*{Acknowledgements}

We thank Vamsi Potluru for lengthy discussions on this work at an early stage.
This work was supported by NSERC Discovery Grant RGPIN-2018-03942.

\bibliography{references}

\begin{thebibliography}{}

\bibitem[Agarwal et~al., 2011]{agarwal2011stochastic}
Agarwal, A., Foster, D.~P., Hsu, D.~J., Kakade, S.~M., and Rakhlin, A. (2011).
\newblock Stochastic convex optimization with bandit feedback.
\newblock In {\em Advances in Neural Information Processing Systems}, pages
  1035--1043.

\bibitem[Auer et~al., 2002]{Auer2002}
Auer, P., Cesa-Bianchi, N., and Fischer, P. (2002).
\newblock Finite-time analysis of the multiarmed bandit problem.
\newblock {\em Machine Learning}, 47(2):235--256.

\bibitem[Badanidiyuru et~al., 2013]{badanidiyuru2013bandits}
Badanidiyuru, A., Kleinberg, R., and Slivkins, A. (2013).
\newblock Bandits with knapsacks.
\newblock In {\em 2013 IEEE 54th Annual Symposium on Foundations of Computer
  Science}, pages 207--216. IEEE.

\bibitem[Bernstein, 1934]{bernstein1934teoriia}
Bernstein, S. (1934).
\newblock Teoriia veroiatnostei [{T}he theory of probabilities].
\newblock {\em Moskva--Leningrad: Gosudarstvennoe Tekhniko-Teoreticheskoe
  Izdatel’stvo.[2nd augmented ed. The 3rd ed., of the same year, is
  identical. The 4th ed., augmented, appeared in 1946.]}.

\bibitem[Blackwell, 1946]{blackwell1946equation}
Blackwell, D. (1946).
\newblock On an equation of {W}ald.
\newblock {\em The Annals of Mathematical Statistics}, 17(1):84--87.

\bibitem[Ding et~al., 2013]{ding2013multi}
Ding, W., Qin, T., Zhang, X.-D., and Liu, T.-Y. (2013).
\newblock Multi-armed bandit with budget constraint and variable costs.
\newblock In {\em Twenty-Seventh AAAI Conference on Artificial Intelligence}.

\bibitem[Flajolet and Jaillet, 2017]{flajolet2015logarithmic}
Flajolet, A. and Jaillet, P. (2017).
\newblock Logarithmic regret bounds for bandits with knapsacks.
\newblock {\em arXiv preprint arXiv:1510.01800v4}.

\bibitem[Gal et~al., 2007]{gal2007participation}
Gal, S., Landsberger, M., and Nemirovski, A. (2007).
\newblock Participation in auctions.
\newblock {\em Games and Economic Behavior}, 60(1):75--103.

\bibitem[Grimmett et~al., 2001]{grimmett2001probability}
Grimmett, G., Grimmett, G.~R., and Stirzaker, D. (2001).
\newblock {\em Probability and random processes}.
\newblock Oxford university press.

\bibitem[Joulani et~al., 2013]{joulani2013online}
Joulani, P., Gyorgy, A., and Szepesv{\'a}ri, C. (2013).
\newblock Online learning under delayed feedback.
\newblock In {\em International Conference on Machine Learning}, pages
  1453--1461.

\bibitem[Komiyama et~al., 2013]{komiyama2013multi}
Komiyama, J., Sato, I., and Nakagawa, H. (2013).
\newblock Multi-armed bandit problem with lock-up periods.
\newblock In {\em Asian Conference on Machine Learning}, pages 100--115.

\bibitem[Lai and Robbins, 1985]{lai1985asymptotically}
Lai, T.~L. and Robbins, H. (1985).
\newblock Asymptotically efficient adaptive allocation rules.
\newblock {\em Advances in Applied Mathematics}, 6(1):4--22.

\bibitem[Lattimore et~al., 2014]{lattimore2014learning}
Lattimore, T., Gy{\"o}rgy, A., and Szepesv{\'a}ri, C. (2014).
\newblock On learning the optimal waiting time.
\newblock In {\em International Conference on Algorithmic Learning Theory},
  pages 200--214. Springer.

\bibitem[Li et~al., 2017]{li2017hyperband}
Li, L., Jamieson, K., DeSalvo, G., Rostamizadeh, A., and Talwalkar, A. (2017).
\newblock Hyperband: A novel bandit-based approach to hyperparameter
  optimization.
\newblock {\em The Journal of Machine Learning Research}, 18(1):6765--6816.

\bibitem[McAfee and McMillan, 1987]{mcafee1987auctions}
McAfee, R.~P. and McMillan, J. (1987).
\newblock Auctions with entry.
\newblock {\em Economics Letters}, 23(4):343--347.

\bibitem[Samuelson, 1985]{samuelson1985competitive}
Samuelson, W.~F. (1985).
\newblock Competitive bidding with entry costs.
\newblock {\em Economics letters}, 17(1-2):53--57.

\bibitem[Stegeman, 1996]{stegeman1996participation}
Stegeman, M. (1996).
\newblock Participation costs and efficient auctions.
\newblock {\em Journal of Economic Theory}, 71(1):228--259.

\bibitem[Tran-Thanh et~al., 2012]{tran2012knapsack}
Tran-Thanh, L., Chapman, A., Rogers, A., and Jennings, N.~R. (2012).
\newblock Knapsack based optimal policies for budget--limited multi--armed
  bandits.
\newblock In {\em Twenty-Sixth AAAI Conference on Artificial Intelligence}.

\bibitem[Weed et~al., 2016]{weed2016online}
Weed, J., Perchet, V., and Rigollet, P. (2016).
\newblock Online learning in repeated auctions.
\newblock In {\em Conference on Learning Theory}, pages 1562--1583.

\bibitem[Xia et~al., 2016]{xia2016budgeted}
Xia, Y., Ding, W., Zhang, X.-D., Yu, N., and Qin, T. (2016).
\newblock Budgeted bandit problems with continuous random costs.
\newblock In {\em Asian Conference on Machine Learning}, pages 317--332.

\end{thebibliography}

\newpage

\onecolumn

\appendix
\section{Proof of Theorem~\ref{thm:exp-cumulative-reward}}
\label{proof_thm1}

The proof of Theorem~\ref{thm:exp-cumulative-reward} relies on the following two fundamental results.

\begin{theorem}[Wald's identity \citep{blackwell1946equation}] \label{thm:wald}
Let $X_1, X_2,\ldots, X_n$ be i.i.d.~random variables having finite mean (so $\E \bigl[ |X_1| \bigr] < \infty$), and $L$ be a stopping time with respect to the filtration $\mathcal{F}_n $ (i.e., $\{L \leq n \} \in \mathcal{F}_n \quad \forall  n \in \mathbb{N}$) satisfying $\E[L] < \infty $. Then
\begin{align*}
\E [ X_1 + \ldots + X_L ] = \E[L] \cdot \E[X_1] .
\end{align*}
\NoEndMark
\end{theorem}

\begin{theorem}[Doob's optional stopping theorem \citep{grimmett2001probability}]
\label{thm:doob}
Let $M_n$ be a martingale with respect to the filtration $\mathcal{F}_n$ and let $L$ be a stopping time. Suppose the following three conditions hold:
\begin{itemize}
\item[(a)] $P(L < \infty) = 1$;
\item[(b)] $\E \bigl[ |M_L| \bigr] < \infty$;
\item[(c)] $\E \bigl[ M_n \cdot \ind{L>n} \bigr] \rightarrow 0$ as $n \rightarrow \infty $.
\end{itemize}
Then $\E [M_L] = \E [M_0]$.
\NoEndMark
\end{theorem}

With the above results at hand, we now prove Theorem~\ref{thm:exp-cumulative-reward}.

\begin{proof}[of Theorem~\ref{thm:exp-cumulative-reward}]

Recall that, for any epoch $s$, we have $\cost{j}{s} = \min\{j, \tau_s\} \leq j$.
Let $w$ be the mean waiting time for micro-arm $j$, so that
\begin{align*}
w := \E [ \cost{j}{1} ] = \E \bigl[ \min\{j, \tau_1\} \bigr] = \sum_{k=1}^{j-1} k \cdot \Pr(\tau_1 = k) + j \cdot \Pr(\tau_1 \geq j) .
\end{align*}
Let $S_n := \sum_{s=1}^n \cost{j}{s}$ be the sum of the waiting times when arm $j$ is pulled for the first $n$ epochs, and let $S_0 := 0$.

Because the waiting times $\cost{j}{s}$ are stochastic, the number of epochs before the game ends (i.e.~before the budget is depleted) is random. Let $L$ be the stopping time with respect to the filtration $(\mathcal{F}_n)_{n \geq 0}$, defined as
\begin{align}\label{lastepoch}
L = \max \{ n: S_n \leq T \} .
\end{align}

Now, from Wald's identity (Theorem~\ref{thm:wald}), the total cumulative reward obtained by the constant policy that always pulls arm $j$ under a time budget $T$ is equal to
\begin{align}\label{eqn:wald}
\E \left[ \sum_{s=1}^L \reward{j}{s} \right] = \E[ \reward{j}{1} ] \cdot \E[L] ,
\end{align}
where we recall that $\reward{j}{s} = V_s \cdot \ind{\tau_s \leq j}$ is the reward for pulling arm $j$ in epoch $s$. 

Define a martingale $(M_n)_{n \geq 0}$ by $M_n := S_n - n \cdot w$. Since each $\cost{j}{s}$ lies in the interval $[1, j]$, the conditions of Doob's optional stopping theorem (Theorem~\ref{thm:doob}) hold, and so
$\E[M_L] = \E[M_0] = 0$. 
Consequently,
we have
\begin{align*}
\E[M_L] = \E[S_L - L \cdot w] = 0 .
\end{align*}
Next, on the one hand, $S_L \leq T$ trivially holds. On the other hand, since $\cost{j}{s} \leq j$ for any epoch $s$, we have $S_L > T - j$ as otherwise $L$ cannot be the last epoch. Therefore
\begin{align*}
T - j < \E [ L ] \cdot w \leq T ,
\end{align*}
and hence
\begin{align}
\frac{T - j}{w} < \E[L] \leq \frac{T}{w} . \label{exp_last_epoch}
\end{align}
Finally, combining \eqref{eqn:wald} and \eqref{exp_last_epoch} yields
\begin{align*}
\E[ \reward{j}{1} ] \cdot \frac{T - j}{w}
< \E \left[ \sum_{s=1}^L \reward{j}{s} \right]
=\E[ \reward{j}{1} ] \cdot \frac{T}{w} ,
\end{align*}
implying the result.
\end{proof}

\section{Proof of Lemma~\ref{lemma:ratio-improved} and an inversion corollary}
\label{app:proof-of-ratio-improved-lemma}

Before proving Lemma~\ref{lemma:ratio-improved}, we first develop and prove a useful supporting lemma.

\begin{lemma} \label{lemma:deviation_inequality:}
Let $Y$ be a random variable taking values in $[1, B]$. Consider a sample $Y_1, \ldots, Y_n$ of independent copies of $Y$. Let $\hat{Y}$ denote the sample mean of $Y_1, \ldots, Y_n$ and $\mu_Y := \E[Y]$.
Take $\delta \in [0, 1]$, Then with probability at least $1 - \delta$ over the sample,
\begin{align}
\left| {\frac{\hat{Y} - \mu_Y}{\mu_Y} }\right|
\leq \sqrt{\frac{(B - 1) \log \frac{2}{\delta}}{2 n}} + \frac{2(B-1) \log \frac{2}{\delta}}{3 \mu_Y n} . \label{eqn:onlybernsteindeviation}
\end{align}
\NoEndMark
\end{lemma}

\begin{proof}
In order to get a tight concentration inequality for our problem, we will be applying Bernstein's inequality.
\\ \\ \textbf{Fact 3} (Bernstein's Inequality \citep{bernstein1934teoriia}) \textit{Assume that $Z_1, \ldots, Z_n$ are centered i.i.d. random variables satisfying $|Z_j| \leq b$ and $\Var[Z_j] \leq \sigma^2$ for all $j \in [n]$. Then}
\begin{align*}
\Pr \left( \frac{1}{n} \sum_{j=1}^n Z_j \geq t \right)
\leq \exp \left( -\frac{n t^2}{2 \left( \sigma^2 + \frac{b}{3} t \right)} \right) .
\end{align*}

We begin by rewriting the left-hand side of \eqref{eqn:onlybernsteindeviation}  as follows:
\begin{align*}
\frac{\left| \hat{Y} - \mu_Y \right|}{\mu_Y}
= \left| \frac{\hat{Y}}{\mu_Y} - 1 \right| .
\end{align*}
Next, observe that
\begin{align*}
\Var[Y] = \E [ Y^2 ] - (\E [ Y ])^2
\leq (B-1)(\mu_Y-1) ,
\end{align*}
where we applied the Bhatia-Davis inequality and used the fact that $\E [ | Y | ] = \E [ Y ] = \mu_Y$. The last inequality may be coarse, but it turns out to simplify some things down the road.

Taking $Z_j = Y_j - \E [ Y_j ]$, we may apply Bernstein's inequality with $b = B - 1$ and $\sigma^2 = (B- \mu_Y)(\mu_Y - 1)$, yielding
\begin{align*}
\Pr \left( \left| \frac{\hat{Y}}{\mu_Y} - 1 \right| \geq t \right)
= \Pr \left( \left| \hat{Y} - \mu_Y \right| \geq \mu_Y t \right)
&\leq 2 \exp \left( -\frac{n \mu_Y^2 t^2}{2 ((\mu_Y-1) (B - 1) + \frac{B-1}{3} \mu_Y t)} \right) \\
&= 2 \exp \left( -\frac{n t^2}{2 (\frac{(B - 1)(\mu_Y-1)}{\mu_Y^2} + \frac{B-1}{3 \mu_Y} t)} \right) .
\end{align*}
Also, we can upper bound $\frac{(B-1)(\mu_Y -1)}{\mu_Y^2}$ by $\frac{(B-1)}{4}$, so that the above probability is at most
\begin{align*}
2 \exp \left( -\frac{n t^2}{2 (\frac{(B - 1)}{4} + \frac{B-1}{3 \mu_Y} t)} \right) .
\end{align*}
We can recover the discrepancy $t$ by inverting the above equation (by setting the failure probability to $\delta$), yielding
\begin{align*}
\log \left( \frac{2}{\delta}\right) = \frac{n t^2}{\frac{(B-1)}{2} + \frac{2(B-1)}{3 \mu_Y} t} \,\,\,\, .
\end{align*}
Solving for t, we have that with probability at least $1 - \delta$,
\begin{align*}
t \leq \sqrt{\frac{(B - 1) \log \frac{2}{\delta}}{2 n}} + \frac{2(B-1) \log \frac{2}{\delta}}{3 \mu_Y n} .
\end{align*}
Therefore, as desired, with probability at least $1 - \delta$,
\begin{align*}
\left| {\frac{\hat{Y} - \mu_Y}{\mu_Y} }\right|
\leq \sqrt{\frac{(B - 1) \log \frac{2}{\delta}}{2 n}} + \frac{2(B-1) \log \frac{2}{\delta}}{3 \mu_Y n} \,\,\,\, .
\end{align*}
\end{proof}

\begin{proof}[of Lemma~\ref{lemma:ratio-improved}]
Define $\mu_X := \E [ X ]$ and $\mu_Y := \E [ Y ]$. We begin with the rewrite
\begin{align*}
\frac{\mu_X}{\mu_Y}
= \frac{\hat{X}}{\mu_Y} + \frac{\mu_X - \hat{X}}{\mu_Y}
= \frac{\hat{X}}{\hat{Y}}
+ \left( \frac{\hat{X}}{\mu_Y} - \frac{\hat{X}}{\hat{Y}} \right)
+ \frac{\mu_X - \hat{X}}{\mu_Y} .
\end{align*}
We bound the second and third terms in turn.

First, observe that
\begin{align*}
\left| \frac{\hat{X}}{\mu_Y} - \frac{\hat{X}}{\hat{Y}} \right|
= \left| \hat{X} \left( \frac{\hat{Y} - \mu_y}{\mu_Y \, \hat{Y}} \right) \right|
\leq \left| \frac{\hat{Y} - \mu_y}{\mu_y} \right| ,
\end{align*}
where the inequality is from the assumptions that $X \in [0, 1]$ and $Y \geq 1$.
Now, from Lemma~\ref{lemma:deviation_inequality:} (stated and proved immediately after this result), we have that with probability at least $1 - \delta$,
\begin{align*}
\left| \frac{\hat{Y} - \mu_Y}{\mu_Y} \right|
\leq \sqrt{\frac{(B - 1) \log \frac{2}{\delta}}{2 n}} + \frac{2(B-1) \log \frac{2}{\delta}}{3 \mu_Y n} .
\end{align*}

Second, again using $Y \geq 1$, we have that
\begin{align*}
\left| \frac{\mu_X - \hat{X}}{\mu_Y} \right|
\leq \left| \mu_X - \hat{X} \right| ;
\end{align*}
this term can be controlled using Hoeffding's inequality, where we now use that $X \in [0, 1]$, this time yielding a deviation of size $\sqrt{\frac{\log \frac{2}{\delta}}{2 n}}$. The result follows from a union bound.
\end{proof}

Next, we state a useful corollary of Lemma~\ref{lemma:ratio-improved}. The setup is identical to that of Lemma~\ref{lemma:ratio-improved}, but we restate the setup for the convenience of the reader. This corollary is simply an inversion of the aforementioned lemma.

\begin{corollary} \label{cor:inversion}
Let $X, Y$ be (possibly dependent) random variables with joint distribution $P$.
Consider a sample $(X_1, Y_1), \ldots, (X_n, Y_n)$ of independent copies of $(X, Y) \sim P$. Assume that $X$ takes values in $[0, 1]$, and $Y$ takes values in $[1, B]$. Define $\mu_Y := \E[Y]$ and let $\hat{X}$ denote the sample mean of $X_1, \ldots, X_n$ (likewise for $\hat{Y}$ and $Y_1, \ldots, Y_n$). For any $\epsilon > 0$, we have
\begin{align}\label{eqn:prob-inversion-new}
\Pr \left( \left| \dfrac{\hat{X}}{\hat{Y}} - \dfrac{\E [X]}{\E [Y]} \right| \geq \epsilon \right) \leq 4 \exp \left( -\left[  \frac{-\frac{(\sqrt{B-1}+1)}{\sqrt{2n}}+ \sqrt{ \left(\frac{(\sqrt{B-1}+1)}{\sqrt{2n}} \right)^2 + \frac{8(B-1) \epsilon}{3 n}} }{\frac{4(B-1)}{3n}} \right]^2 \right)
\end{align}
\NoEndMark
\end{corollary}
The proof is by inversion and simply involves solving a quadratic equation.

\section{Proof of Lemma~\ref{lemma:exp-sub-pulls}}
\label{app:proof-of-exp-sub-pulls-lemma}

For convenience, we recall a quantity defined in the main text which will be used in the proof.
For all $(k,j) \in [K] \times [D]$ and any $s \geq 1$,
\begin{align*}
\updev{k}{j}{s} = \alpha_j \frac{\log s}{\npulls{k,j}{s}} + \beta_j \sqrt{\frac{\log s}{\npulls{k,j}{s}}} .
\end{align*}

\begin{proof}[of Lemma~\ref{lemma:exp-sub-pulls}]

Let $\gtrue{k}{j}$ be the ratio of expected reward to expected waiting time for a pull of arm pair $(k,j)$, and let $(k^*,j^*)$ be the optimal arm pair, so that $\gtrue{k^*}{j^*} =\max_{(k,j) \in [K] \times [D]} \gtrue{k}{j}$. Let $\npulls{k,j}{s}$ be the number of pulls of arm pair $(k,j)$ until the end of epoch $s$. It will be useful to (implicitly) define a function $\hat{h}_{k,j}(\cdot)$ as $\hhat{k}{j}{\npulls{k,j}{s}} := \ghat{k}{j}{s}$; this function gives the empirical reward per round of arm pair $(k,j)$ for $\npulls{k,j}{s}$ pulls. Let $\confrad{k}{j}{\npulls{k,j}{s}} := a_{k,j}(s)$ be the confidence radius of arm $(k,j)$ for $\npulls{k,j}{s}$ pulls.
Let $i_s$ denote the arm pair pulled in epoch $s$. Recall that $L$ is the stopping time for the game.
The number of pulls of suboptimal arm pair $(k,j)$ with $\gap{k}{j} > 0$ in $L$ epochs is
\begin{align}
\npulls{k,j}{L}  =  \sum_{s=1}^L \ind{ i_s =  (k,j) } . \label{eqn:njs}
\end{align}
Since the time consumed in each epoch is stochastic (depending on the delay random variable $\tau_s \sim Q$), we first upper bound the random stopping time $L$ by $T$; this is possible because each epoch lasts for at least one round.
This upper bound, combined with the fact that each micro-arm is pulled once in the first $D$ rounds, implies that \eqref{eqn:njs} is at most
\begin{align*}
\sum_{s=1}^T \ind{i_s= (k,j)}
= 1 + \sum_{s=D+1}^T \ind{i_s= (k,j)} .
\end{align*}
Let $l$ be an arbitrary integer. We proceed by decomposing the second term into two sampling regimes. When $\npulls{k,j}{L} < l$, we say that the sub-optimal arm $(k,j)$ is in the under-sampled regime and if $\npulls{k,j}{L} \geq l$, we say that the sub-optimal arm $(k,j)$ is in the sufficiently sampled regime. The tuning of the value of $l$ is given in detail in the proof of Lemma~\ref{lemma:suff-sample}. Also, we show that when $(k,j)$ is in the sufficiently sampled regime, we can use Corollary~\ref{cor:inversion}. Now, the summation in the RHS of the last line above is equal to
\begin{align}
&\sum_{s=D+1}^T \ind{ i_s = (k,j), \npulls{k,j}{s-1} < l }
   + \sum_{s=D+1}^T \ind{ i_s = (k,j), \npulls{k,j}{s-1} \geq l } \nonumber \\
& \leq  l + \sum_{s=D+1}^T \ind{ i_s = (k,j); \npulls{k,j}{s-1} \geq l } \nonumber \\
& \leq  l + \sum_{s=D+1}^T
                      \ind{ \ghat{k}{j}{s-1} + \updev{k}{j}{s-1} \geq \ghat{k^*}{j^*}{s-1} + \updev{k^*}{j^*}{s-1};
                                \npulls{k,j}{s-1} \geq l } \nonumber \\
& \leq l + \sum_{s=D+1}^T
                     \ind{ \max_{l \leq p < s} \left\{ \hhat{k}{j}{p} + \confrad{k}{j}{p} \right\}
                              \geq  \min_{0 < m < s} \left\{ \hhat{k^*}{j^*}{m} + \confrad{k^*}{j^*}{m} \right\}
                            } \nonumber \\
& \leq l + \sum_{s=D+1}^T \sum_{m=1}^{s-1} \sum_{p=l}^{s-1}
                     \ind{ \hhat{k}{j}{p} + \confrad{k}{j}{p}  \geq  \hhat{k^*}{j^*}{m} + \confrad{k^*}{j^*}{m} } \nonumber \\
& \leq l + \sum_{s=1}^T \sum_{m=1}^{s-1} \sum_{p=l}^{s-1}
                     \ind{ \hhat{k}{j}{p} + \confrad{k}{j}{p}  \geq  \hhat{k^*}{j^*}{m} + \confrad{k^*}{j^*}{m} } \label{arm_pull_cond} .
\end{align}

Now, in \eqref{arm_pull_cond}, the inequality
$\hhat{k}{j}{p} + \confrad{k}{j}{p}  \geq  \hhat{k^*}{j^*}{m} + \confrad{k^*}{j^*}{m}$ is possible only when at least one of the following three inequalities is true:
%
\begin{align}\label{underestimation}
\hhat{k^*}{j^*}{m} + \confrad{k^*}{j^*}{m} \leq \gtrue{k^*}{j^*} ;
\end{align}
\begin{align}\label{overestimation}
\hhat{k}{j}{p} \geq \gtrue{k}{j} +\confrad{k}{j}{p} ;
\end{align}
\begin{align}\label{sufficient_sample}
\gtrue{k^*}{j^*} < \gtrue{k}{j} + 2 \confrad{k}{j}{p} .
\end{align}
Inequality \eqref{underestimation} corresponds to $\hhat{k^*}{j^*}{m}$ being a significant underestimate of the optimal ratio $\gtrue{k^*}{j^*}$,
while inequality \eqref{overestimation} corresponds to $\hhat{k}{j}{p}$ being a significant overestimate of the suboptimal ratio $\gtrue{k}{j}$.
Finally, inequality \eqref{sufficient_sample} will turn out to be false provided that $l$ is selected to be large enough, as then, from $p \geq l$, the quantity $\confrad{k}{j}{p}$ will be small enough to be strictly less than $\gap{k}{j}/2$. Refer to the proof of Lemma \ref{lemma:suff-sample} for the selection of $l$.

Taking the expectation on both sides of \eqref{arm_pull_cond}, we have
\begin{align}
\E[\npulls{k,s}{L}] &\leq l + \sum_{s=1}^T \sum_{m=1}^{s-1} \sum_{p=l}^{s-1} \Pr ( \hhat{k}{j}{p} + \confrad{k}{j}{p}  \geq  \hhat{k^*}{j^*}{m} + \confrad{k^*}{j^*}{m} ) \nonumber \\
& \leq l + \sum_{s=1}^T \sum_{m=1}^{s-1} \sum_{p=l}^{s-1}  \Pr ( \hhat{k}{j}{p} \geq \gtrue{k}{j}  + \confrad{k}{j}{p} ) + \Pr (  \hhat{k^*}{j^*}{m} \leq \gtrue{k^*}{j^*} - \confrad{k^*}{j^*}{m} ) \label{prob_arm_pull_cond} .
\end{align}
We bound the above probabilities using Corollary \ref{cor:inversion}; note that the time consumed by pulling an arm pair $(k,j)$ will be at most $j$, so $B$ in Corollary \ref{cor:inversion} becomes $j$ and $\epsilon$ becomes $\confrad{k}{j}{s} = \alpha_j \frac{\log s}{\npulls{k,j}{s}} + \beta_j \sqrt{\frac{\log s}{\npulls{k,j}{s}}}$.\\ For $\alpha_j = \frac{8(j-1)}{3}$, $ \beta_j = \sqrt{2}(\sqrt{j-1}+1)$ and by setting $l =  {\left[ \frac{\beta_j \sqrt{\log T} + \sqrt{\beta_j^2 \log T + 2 \gap{k}{j} \alpha_j \log T}}{\gap{k}{j}} \right]}^2 $, the above probabilities are bounded as
\begin{align} \label{bounding_prob}
&\Pr( \hhat{k}{j}{p} \geq \gtrue{k}{j}  + \confrad{k}{j}{p} ) \leq 4 \exp(-4 \log s) = 4 s^{-4} \nonumber ; \\
&\Pr( \hhat{k^*}{j^*}{m} \leq \gtrue{k}{j^*} - \confrad{k^*}{j^*}{m} ) \leq 4 s^{-4} .
\end{align}

We explain in detail how we chose $\alpha_j$, $\beta_j$, and $l$ in the proof of Lemma~\ref{lemma:suff-sample}. 

Using the above bounds in \eqref{prob_arm_pull_cond}, we have
\begin{align*}
\E[\npulls{k,j}{L}] 
&\leq l + 8 \sum_{s=1}^T \sum_{m=l}^{s-1} \sum_{p=l}^{s-1}  s^{-4} \\
&\leq l + 8 \sum_{s=1}^{T} s^{-2} \\
&\leq l + 8 \sum_{s=1}^{\infty} s^{-2} \\
&\leq l + 8 \cdot \frac{{\pi}^2}{6} \\
&= l + 4 \cdot \frac{{\pi}^2}{3} \\
& = \left( \frac{\sqrt{2}(\sqrt{j-1}+1) \sqrt{\log T})+ \sqrt{2(\sqrt{j-1}+1)^2 \log T + \frac{16}{3} \Delta_j (j-1) \log T }}{\Delta_j} \right)^2 + 4 \cdot \dfrac{{\pi}^2}{3} ,
\end{align*}
where the last line is from Lemma~\ref{lemma:suff-sample}, thus completing the proof.
\end{proof}

\section{Proof of Lemma~\ref{lemma:suff-sample}}
\label{app:proof-of-suff-sample-lemma}

In the following proof, since $j$ is fixed throughout, we simply write $\alpha$ and $\beta$ instead of $\alpha_j$ and $\beta_j$ (i.e.~we drop the subscripts).
\begin{proof}
We now find a value of $l$ that marks the sufficient sampling regime (this regime was described in the proof of Lemma~\ref{lemma:exp-sub-pulls}). We know that inequality \eqref{sufficient_sample} is only true for $p < l$. 
We will use this argument to find the lowerbound on $l$. 
Note that the value of $l$ is different for different pairs of arms; we used $l$ for notation consistency with $p$. 
The following three inequalities are equivalent:
\begin{align*}
\confrad{k}{j}{p} & < \frac{\gap{k}{j}}{2} ; \\
\alpha \frac{\log s}{p} + \beta \sqrt{\frac{\log s}{p}} &<  \frac{\gap{k}{j}}{2} ; \\
\alpha \log s + \beta \sqrt{p \log s} & < p \frac{\gap{k}{j}}{2} . \\
\end{align*}
The last line above is quadratic in $Z = \sqrt{p}$. 
Take 
$a = \frac{-\gap{k}{j}}{2}$, $
b = \beta \sqrt{\log s}$, 
and $c = \alpha \log s$ to be the coefficients for the quadratic form. Then we see that
\begin{align*}
Z < \frac{-\beta \sqrt{\log s} - \sqrt{\beta^2 \log s + 2 \gap{k}{j} \alpha \log s}}{-\gap{k}{j}} \\
\end{align*}
and hence
\begin{align*}
p  < {\left[ \frac{\beta \sqrt{\log s} + \sqrt{\beta^2 \log s + 2 \gap{k}{j} \alpha \log s}}{\gap{k}{j}} \right]}^2 ,
\end{align*}
so, for $p \geq {\left[ \frac{\beta \sqrt{\log s} + \sqrt{\beta^2 \log s + 2 \gap{k}{j} \alpha \log s}}{\gap{k}{j}} \right]}^2$, inequality \eqref{sufficient_sample} is false. Therefore,
\begin{align*}
l \geq {\left[ \frac{\beta \sqrt{\log T} + \sqrt{\beta^2 \log T + 2 \gap{k}{j} \alpha \log T}}{\gap{k}{j}} \right]}^2 .
\end{align*}
To get the values of $\alpha_j$ and $\beta_j$ to bound the probabilities as in \eqref{bounding_prob}, 
we start by equating the probability quantity obtained from Corollary~\ref{cor:inversion} to $4 \exp( -4 \log s)$, 
so that
\begin{align*}
4 \exp( -4 \log s) &= 4 \exp \left( -\left[  \frac{-\frac{(\sqrt{j-1}+1)}{\sqrt{2n}}+ \sqrt{ \left(\frac{(\sqrt{j-1}+1)}{\sqrt{2n}} \right)^2 + \frac{8(j-1) \epsilon}{3 n}} }{\frac{4(j-1)}{3n}} \right]^2 \right) \\
4 \log s &= \left[  \frac{-\frac{(\sqrt{j-1}+1)}{\sqrt{2n}}+ \sqrt{ \left(\frac{(\sqrt{j-1}+1)}{\sqrt{2n}} \right)^2 + \frac{8(j-1) \epsilon}{3 n}} }{\frac{4(j-1)}{3n}} \right]^2 .
\end{align*}
For the equation \eqref{bounding_prob}, $\epsilon = \alpha \frac{\log s}{n} + \beta \sqrt{\frac{\log s}{n}}$. We simplify the above equation,
\begin{align*}
& -\frac{(\sqrt{j-1}+1)}{\sqrt{2n}}+ \sqrt{ \left(\frac{(\sqrt{j-1}+1)}{\sqrt{2n}} \right)^2 + \frac{8(j-1) \epsilon}{3 n}} = 2 \sqrt{\log s} \left( \frac{4(j-1)}{3n} \right) .
\end{align*}
Squaring both sides yields
\begin{align*} 
\left(\frac{(\sqrt{j-1}+1)}{\sqrt{2n}} \right)^2 + \frac{8(j-1) \epsilon}{3 n} = \left( 2 \sqrt{\log s} \left( \frac{4(j-1)}{3n} \right) + \frac{(\sqrt{j-1}+1)}{\sqrt{2n}} \right)^2 + \frac{(\sqrt{j-1}+1)^2}{2n} .
\end{align*}
Extracting the terms on the RHS, we have
\begin{align*}
& \frac{\left(\sqrt{j-1}+1\right)^2}{2} + \frac{8(j-1) \epsilon}{3} =  4 \log s \frac{16(j-1)^2}{9n} + (16 \sqrt{\log s}) \left( \frac{(j-1)}{3} \right) \left( \frac{\sqrt{j-1}+1}{\sqrt{2n}} \right) + \frac{(\sqrt{j-1}+1)^2}{2} .
\end{align*}
Substituting $\epsilon$ in the above equation yields
{\small
\begin{align*}\label{eqn:tuning_alpha_beta}
& \frac{\left(\sqrt{j-1}+1\right)^2}{2} + \frac{8(j-1) \{ \alpha \frac{\log s}{n} + \beta \sqrt{\frac{\log s}{n}} \}  }{3} =  4 \log s \frac{16(j-1)^2}{9n} + (16 \sqrt{\log s}) \left( \frac{(j-1)}{3} \right) \left( \frac{\sqrt{j-1}+1}{\sqrt{2n}} \right)+ \frac{(\sqrt{j-1}+1)^2}{2} .
\end{align*}
}
Simplifying the above equation, we get
\begin{align}
& \frac{\log s}{n} \left[ \frac{8(j-1)}{3} (\alpha) - \frac{64(j-1)^2}{9} \right] + \frac{\log s}{n} \left[\frac{8(j-1) (\beta)}{3} - \frac{16(j-1)}{3 \sqrt{2}}[\sqrt{j-1}+1] \right] = 0 .
\end{align}

The above equation \eqref{eqn:tuning_alpha_beta} is only true for $\alpha = \frac{8(j-1)}{3}$ and $\beta=\sqrt{2}(\sqrt{j-1}+1)$. This helps us bound $\E[\npulls{k,j}{s}]$.
\end{proof}

\section*{Proof of Theorem~\ref{regret}}
\begin{proof}
The proof mainly uses the Lemma~\ref{lemma:exp-sub-pulls}
\begin{align}
R_T & = T \cdot \gtrue{k^*}{j^*} - \E \left[ \sum_{s=1}^{L} r_{i_s}^{(s)} \right] + \mathcal{O}(1) \nonumber \\
& \leq T \cdot \gtrue{k^*}{j^*} - \sum_{(k,j)} \mu_{k,j}^{(r)} \cdot \E[\npulls{k,j}{L}] + \mathcal{O}(1) \nonumber \\
& = \gtrue{k^*}{j^*} \cdot (T - \sum_{(k,j) | \gap{k}{j} = 0} \mu_{k,j}^{(c)} \cdot \E[\npulls{k,j}{L}]) - \sum_{(k,j) | \gap{k}{j} > 0} \mu_{k,j}^{(r)} \cdot \E[\npulls{k,j}{L}]) \label{regretwitht}
\end{align}
Taking account of the random stopping time $L$, we can write,
\begin{align*}
T \leq \sum_{t=1}^{L} \cost{i_s}{s} .
\end{align*}
Taking expectations on both sides, we get
\begin{align}\label{sum_of_exp_cost}
T & \leq \sum_{k=1}^K \sum_{j=1}^D \mu_{k,j}^{(c)} \cdot \E[\npulls{k,j}{L}] \nonumber \\
& = \sum_{(k,j)| \gap{k}{j}=0} \mu_{k,j}^{(c)} \cdot \E[\npulls{k,j}{L}] + \sum_{(k,j)| \gap{k}{j} > 0} \mu_{k,j}^{(c)} \cdot \E[\npulls{k,j}{L}].
\end{align}
Substituting the above inequality \eqref{sum_of_exp_cost} in the regret bound \eqref{regretwitht}, we get
\begin{align*}
R_T &\leq \gtrue{k^*}{j^*} \cdot \left( \sum_{(k,j) | \gap{k}{j} > 0} \mu_{k,j}^{(c)} \cdot \E[\npulls{k,j}{L}] \right) - \sum_{(k,j) | \gap{k}{j} > 0} \mu_{k,j}^{(r)} \cdot \E[\npulls{k,j}{L}]) + \mathcal{O}(1) \\
& = \sum_{(k,j) | \gap{k}{j} > 0} \mu_{k,j}^{(c)} \left( \gtrue{k^*}{j^*}  - \gtrue{k}{j} \right)\E[\npulls{k,j}{L}] + \mathcal{O}(1) \\
& = \sum_{(k,j) | \gap{k}{j} > 0} \mu_{k,j}^{(c)} \left( \gap{k}{j} \right)\E[\npulls{k,j}{L}] + \mathcal{O}(1). \\
\end{align*}
From Lemma~\ref{lemma:exp-sub-pulls}, we have an upper bound on $\E[\npulls{k,j}{L}]$, and so the regret $R_T$ is at most
\begin{align*}
\sum_{(k,j) | \gap{k}{j} > 0} \mu_{k,j}^{(c)} \left(  \frac{ \left( \sqrt{2}(\sqrt{j-1}+1) \sqrt{\log T})+ \sqrt{2(\sqrt{j-1}+1)^2 \log T + \frac{16}{3} \gap{k}{j} (j-1) \log T } \right)^2 }{\gap{k}{j}}  + 4 (\dfrac{{\pi}^2}{3})\gap{k}{j}   \right) + \mathcal{O}(1)
\end{align*}
\NoEndMark
\end{proof}

\end{document}